%% file: acl_latex.tex
\documentclass[11pt]{article}

\usepackage[final]{acl}

\usepackage{times}
\usepackage{latexsym}

\usepackage[T1]{fontenc}

\usepackage[utf8]{inputenc}

\usepackage{microtype}
\usepackage[table]{xcolor}  
\usepackage{inconsolata}
\usepackage[ruled,vlined]{algorithm2e}
\usepackage{float}
\usepackage{tabularx}
\usepackage{multirow}
\usepackage{graphicx}
\usepackage{pifont}
\usepackage{amsmath,amssymb,amsfonts}
\usepackage{amsthm}
\usepackage{booktabs} 
\newtheorem{theorem}{Theorem}
\newtheorem{lemma}{Lemma}
\newtheorem*{theorem*}{Theorem}

\usepackage{arydshln}
\usepackage{makecell}
\usepackage{subcaption}
\usepackage{xcolor}
\newcommand{\tablestyle}[2]{\setlength{\tabcolsep}{#1}\renewcommand{\arraystretch}{#2}\centering\footnotesize}

\title{BWLA: Breaking the Barrier of W1AX Post-Training Quantization for LLMs}

\author{
  \textbf{Zhixiong Zhao}\textsuperscript{1}\thanks{~~Equal contribution.}\thanks{~~This work was conducted during his internship at Houmo AI.} ,
  \textbf{Zukang Xu}\textsuperscript{1}\footnotemark[1] ,
  \textbf{Dawei Yang}\textsuperscript{1}\thanks{~~Corresponding author.} \\
  \\
  \textsuperscript{1}Houmo AI\\
  \small{\textbf{Correspondence:} \href{mailto:zukang.xu@houmo.ai}{zukang.xu@houmo.ai}, \href{mailto:dawei.yang@houmo.ai}{dawei.yang@houmo.ai}}
}

\begin{document}
\maketitle

\input{sec/0_abstract}
\input{sec/1_intro}
\input{sec/2_related-work}
\input{sec/3_method}
\input{sec/4_experiments}
\input{sec/5_conclusion}

\clearpage
\section*{Limitations}
While BWLA establishes a new state of the art for high-accuracy W1AX PTQ, several limitations remain. First, although the method maintains stable performance at W1A6, model stability noticeably degrades at more extreme activation bit-widths such as W1A4, indicating that current activation smoothing and outlier suppression are insufficient for this challenging regime. Second, the core OKT module relies on linear orthogonal rotations, which, despite their efficiency, may be too restrictive to fully capture the non-linear geometry of modern LLM weight spaces; lightweight non-linear transformations may offer further gains. Finally, BWLA currently targets standard integer formats (e.g., INT4/INT6), and extending the optimization to mixed-precision or emerging low-precision floating-point formats such as MXFP4 could broaden applicability and improve dynamic-range handling.

\section*{Ethics Statements}
This paper introduces solutions to the challenges associated with Large Language Models (LLMs) quantization, with the overarching goal of facilitating the widespread adoption and application of LLMs. In the current landscape, ethical concerns tied to LLMs, including the presence of hidden biases encoded in the models, are garnering heightened attention. Following our investigation, we assert that our proposed method does not further amplify the biases and contravene any ethical standards.

\bibliography{custom}

\clearpage
\appendix

\section{Appendix}
\label{sec:appendix}

\input{appendix/A.1}
\input{appendix/A.2}
\input{appendix/A.3}

\input{appendix/A.4}

\end{document}

%% file: sec/0_abstract.tex
\begin{abstract}
Large language models (LLMs) have driven major progress in NLP, yet their substantial memory and compute demands still hinder practical deployment. Binarization can compress weights to 1 bit, fundamentally lowering compute and bandwidth cost. However, existing methods cannot address activation heavy tails and thus must keep activations in high precision, preventing true end-to-end acceleration. To overcome this limitation, we propose \textbf{BWLA} (\textbf{B}inarized \textbf{W}eights and \textbf{L}ow-bit \textbf{A}ctivations), the first post-training quantization framework that preserves high accuracy while achieving 1-bit weight quantization together with low-bit activations (e.g., 6 bits). The Orthogonal-Kronecker Transformation (OKT) learns an orthogonal mapping via EM minimization, converting unimodal weights into symmetric bimodal forms while suppressing activation tails and incoherence. The Proximal SVD Projection (PSP) then performs lightweight low-rank refinement through proximal SVD projection, further enhancing quantizability with minimal overhead. On Qwen3-32B, BWLA reaches a Wikitext2 perplexity of \textbf{11.92} under 6-bit activations (vs. 38 from SOTA), improves five zero-shot tasks by more than \textbf{70\%}, and delivers \textbf{3.26×} inference speedup, demonstrating strong potential for real-world LLM compression and acceleration. The code will be available at \href{https://github.com/Kishon-zzx/BWLA}{BWLA}.
\end{abstract}

%% file: sec/1_intro.tex
\section{Introduction}
In recent years, Transformer-based large language models have achieved substantial progress across a wide range of natural language processing tasks. Much of this improvement is driven by the continuous scaling of model size, with many state-of-the-art models reaching tens or even hundreds of billions of parameters. For example, the LLaMA family~\citep{llama3herdmodels} includes models of various scales, with LLaMA2-70B containing 70 billion parameters and requiring more than 120 GB of memory for inference at FP16 precision. Such resource requirements make deployment on mobile devices and other resource-constrained platforms highly challenging.

\begin{figure}
    \centering
    \includegraphics[width=0.9\linewidth]{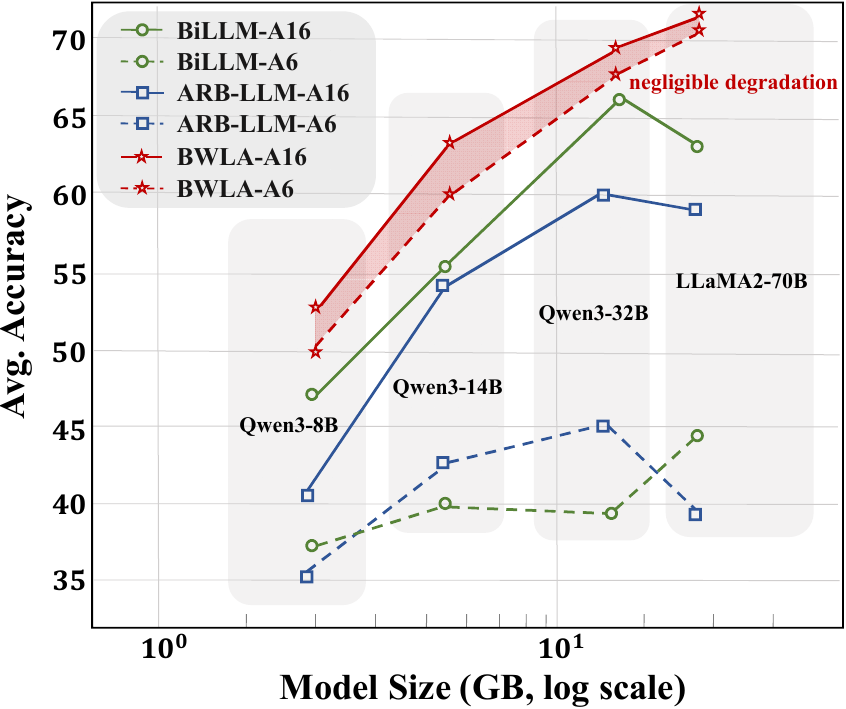}
    \caption{Performance of BWLA versus state-of-the-art binarization methods on five zero-shot QA benchmarks. BWLA remains robust under both weight-only and weight–activation quantization, while other methods degrade sharply with activation quantization.}
    \label{fig1}
\end{figure}

\begin{figure*}[t!]
    \centering
    \includegraphics[width=0.92\linewidth]{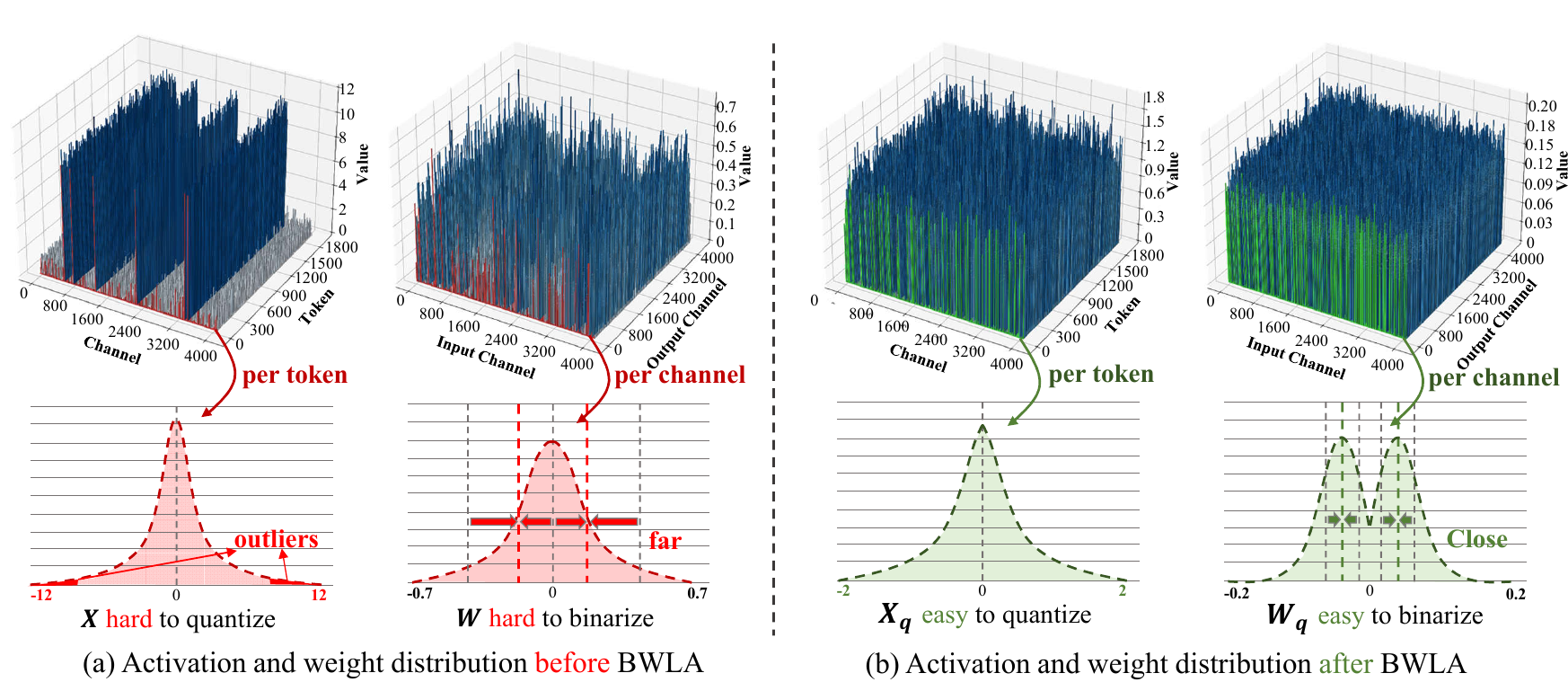}
    \caption{(a) Before applying BWLA, activations contain substantial outliers that hinder low-bit quantization, and unimodal weight distributions are poorly aligned with binarization. (b) After applying BWLA, activations are smoothed with far fewer outliers, facilitating low-bit quantization, and weight distributions are transformed into symmetric bimodal forms that better support binarization.}
    \label{fig2}
\end{figure*}

Quantization has become a central approach to model compression across various large-scale architectures, encompassing both dense models~\citep{zhao2025quark,zhao2026specquant} and Mixture-of-Experts (MoE) networks~\citep{xu2026kbvq}. Within this paradigm, binarization is particularly appealing because it reduces weight storage to a single bit. Recent studies such as BiLLM~\citep{huang2024billmpushinglimitposttraining} and ARB-LLM~\citep{li2024arbllmalternatingrefinedbinarizations} leverage local Hessians to estimate weight saliency and apply refined handling to salient weights, thereby mitigating the performance degradation introduced by binarization. However, compared with weights, activations are more prone to outliers and heavy-tailed distributions~\citep{ashkboos2024quarotoutlierfree4bitinference}. Existing binarization methods predominantly focus on weight-only designs while overlooking the quantization characteristics of activations, leading to suboptimal performance when activations are quantized to low bit-widths (Figure~\ref{fig1}). As a consequence, such methods often retain full-precision activations to avoid numerical distortion, which necessitates weight dequantization during inference. This design incurs substantial computational and memory overhead, thereby limiting the achievable inference acceleration. BitNet a4.8~\citep{wang2024bitneta484bitactivations} addresses this issue by using resource-intensive quantization-aware training (QAT) to achieve 1-bit weights with 4-bit activations. These limitations underscore the need for post-training quantization (PTQ) methods that jointly consider weights and activations, aiming to minimize performance loss without retraining and to enable efficient end-to-end inference. Achieving 1-bit weights together with low-bit activations \textbf{(W1AX)}, while preserving model quality remains an important and technically demanding challenge.

We revisit the statistical characteristics of weights and activations in LLMs, as illustrated in Figure~\ref{fig2} (a). These characteristics present a fundamental challenge to W1AX: (1) Per-channel weight distributions are often unimodal (quasi-Gaussian), severely mismatched with the $\pm 1$ codebook, leading to substantial binarization error. (2) Activations exhibit pronounced heavy-tailed behavior with extreme outliers, which dominates quantization distortion at low bit-widths. These observations motivate a central question: \emph{within a post-training quantization framework, can we jointly \ding{182} reshape per-channel weight distributions into a bimodal form that better suits binarization, and \ding{183} suppress activation outliers to mitigate heavy tails, thereby facilitating both binary weight quantization and low-bit activation quantization?}

To this end, we propose \textbf{BWLA} (\textbf{B}inarized \textbf{W}eights and \textbf{L}ow-bit \textbf{A}ctivations), a post-training framework that jointly quantizes weights and activations, enabling near-1-bit weights and low-bit activations without end-to-end training. The core idea is to iteratively optimize an orthogonal transformation that reshapes unimodal weight distributions into symmetric bimodal forms. Leveraging the property $R^{-1} = R^\top$, the same transformation can be applied to activations while preserving forward-pass equivalence, thereby effectively suppressing activation outliers. Specifically, the \textbf{Orthogonal–Kronecker Transformation (OKT)} formulates a gradient-free objective that encourages bimodal clustering and solves the orthogonal matrix via EM-style conditional minimization. OKT constructs the transformation through Kronecker factorization into two small orthogonal matrices, incurring negligible additional computation and memory overhead in the PTQ setting. \textbf{Proximal SVD Projection (PSP)} applies a lightweight low-rank correction in the aligned coordinate space.
It absorbs remaining outliers through a proximal upper bound and truncated SVD projection, reinforcing bimodality with minimal parameter overhead. As shown in Figure~\ref{fig2}(b), BWLA transforms unimodal weights into nearly symmetric bimodal distributions and suppress activations outliers, achieving unified PTQ with near-1-bit weights and low-bit activations.

To summarize, our main contributions are:
\begin{itemize}
    \item We identify the W1AX bottleneck in LLMs: weight–codebook mismatch and heavy-tailed activations fundamentally limit existing PTQ methods from jointly reducing weight and activation precision.
    \item We propose BWLA, the first retraining-free PTQ framework that achieves 1-bit weights and low-bit activations quantization.
    \item Extensive experiments show that BWLA outperforms prior methods across multiple LLMs, achieving the first high-accuracy W1AX under pure PTQ.
\end{itemize}

%% file: sec/2_related-work.tex
\section{Related Work}
\paragraph{Post Training Quantization for LLMs}
Post-training quantization (PTQ) has become a mainstream technique for compressing and accelerating large language models due to its low calibration cost and lack of retraining. Existing methods are generally categorized into weight-only and joint weight–activation quantization.
Weight-only approaches reduce memory usage by quantizing parameters alone. GPTQ~\citep{frantar2023gptqaccurateposttrainingquantization} applies Hessian-guided error compensation, AWQ~\citep{lin2024awqactivationawareweightquantization} mitigates activation outliers, and QuIP\#~\citep{tseng2024quipbetterllmquantization} combines random Hadamard transforms with vector quantization to retain accuracy at low bit widths.
Joint weight–activation quantization enables end-to-end acceleration by compressing both weights and activations. SmoothQuant~\citep{xiao2024smoothquantaccurateefficientposttraining} transfers quantization difficulty from activations to weights via reversible scaling, while OmniQuant~\citep{shao2024omniquantomnidirectionallycalibratedquantization} jointly calibrates transformations and quantizers. Recent methods such as QuaRot~\citep{ashkboos2024quarotoutlierfree4bitinference}, SpinQuant~\citep{liu2025spinquantllmquantizationlearned}, and OSTQuant~\citep{hu2025ostquantrefininglargelanguage} further enhance robustness through random or learnable rotations. Nevertheless, unimodal weight distributions and heavy-tailed activations remain fundamental challenges for efficient low-bit PTQ.

\paragraph{Binarization for LLMs}
Binarization is an extreme low-bit technique that compresses model weights to binary values (e.g., -1/+1), substantially reducing memory and computation. However, the accuracy sensitivity of LLMs makes direct binarization challenging, particularly for attention blocks and large embedding layers. BinaryBERT~\citep{bai2021binarybertpushinglimitbert} alleviates degradation by retaining selected high-precision weights, while PB-LLM~\citep{shang2023pbllmpartiallybinarizedlarge} adopts partial binarization, preserving salient weights in full precision to maintain reasoning ability.
Recent work explores structured binarization, leveraging sparsity patterns and normalized importance scores for selective binarization and sparsification~\citep{huang2024billmpushinglimitposttraining}, as well as alternating refinement or column-group bitmaps to reduce quantization error and column bias~\citep{li2024arbllmalternatingrefinedbinarizations}. These advances improve the practicality of binarization and push the efficiency–accuracy frontier. Nevertheless, most methods remain weight-only and cannot achieve end-to-end acceleration. DBellQuant~\citep{ye2025dbellquantbreakingbelldoublebell} moves toward joint weight–activation quantization but relies mainly on learnable scaling and offers limited suppression of activation outliers. Achieving W1AX—1-bit weights with low-bit activations—under retraining-free PTQ therefore remains challenging, requiring simultaneous control of activation outliers, reduced activation precision, and accurate weight quantization.

%% file: sec/3_method.tex
\section{Method}

\begin{figure*}[htbp]
    \centering
    \includegraphics[width=0.95\linewidth]{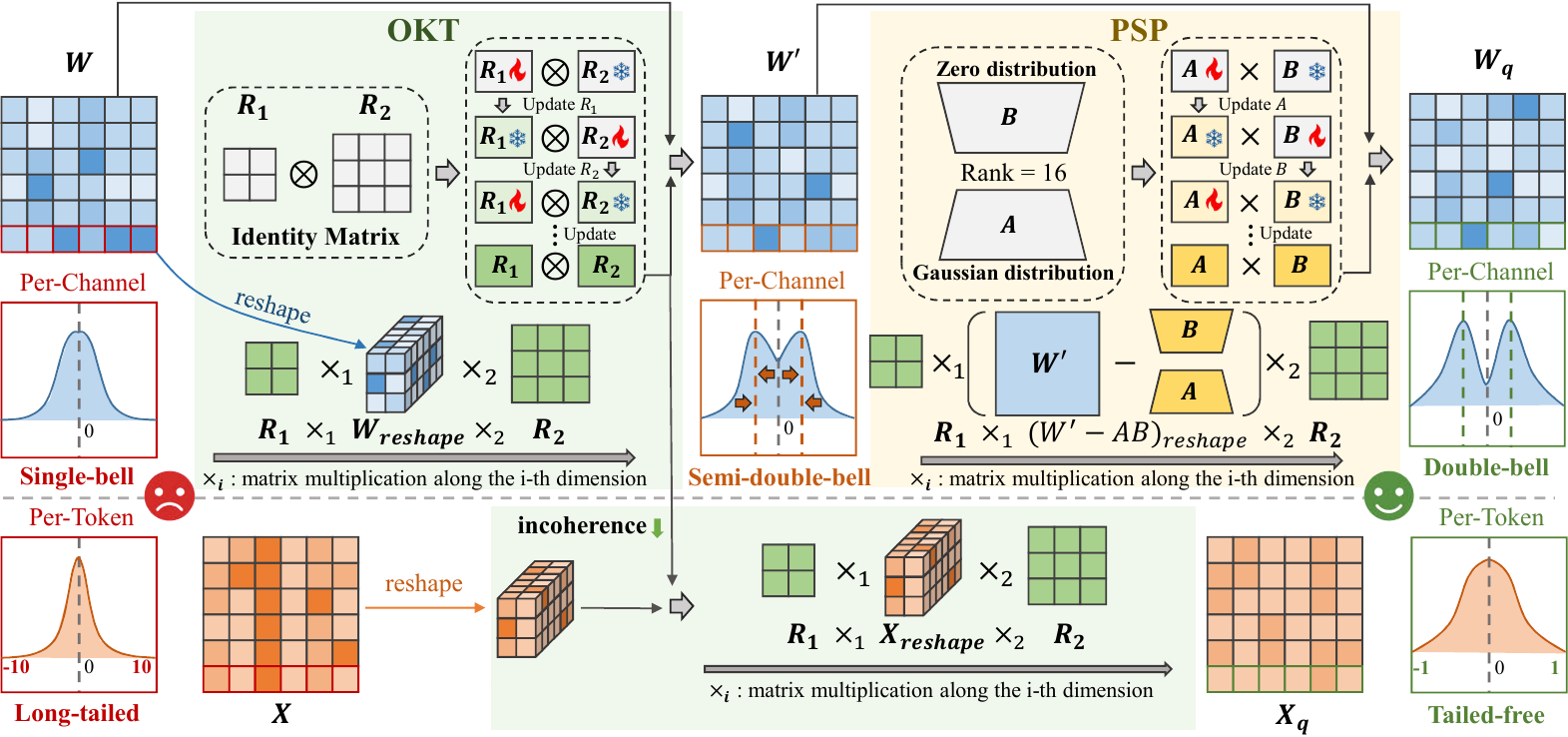}
    \caption{Illustration of the proposed BWLA. The \textbf{Orthogonal-Kronecker Transformation (OKT)} applies an orthogonal Kronecker rotation to reshape weights into a symmetric bimodal space while jointly suppressing long-tailed activation outliers. The \textbf{Proximal SVD Projection (PSP)} further strengthens the bimodal structure through a lightweight truncated SVD refinement, producing weight distributions explicitly optimized for binarization.}
    \label{fig:overview}
\end{figure*}

\paragraph{Discussion.}
We begin by revisiting the standard formulation of weight binarization in LLMs. 
Given a weight matrix $W \in \mathbb{R}^{n \times m}$, since the mean of each output channel (row) is generally non-zero, binarization typically applies row-wise centering prior to quantization. 
For the $i$-th output channel (row), let $\mu_i$ and $\delta_i$ denote the shift and scaling factors, respectively. The centered and binarized weights are computed as:
\begin{align}
\mu_i = \frac{1}{m}\sum_{j=1}^{m} W_{i,j},& \quad \delta_i = \frac{1}{m}\sum_{j=1}^{m} \bigl|W_{i,j} - \mu_i\bigr|,\\
\tilde{W}_{i,j} &= \operatorname{Sign}\!\bigl(W_{i,j} - \mu_i\bigr), 
\end{align}
where
\begin{equation}
\operatorname{Sign}(x) =
\begin{cases}
+1, & x > 0,\\[2pt]
-1, & x \le 0.
\end{cases}
\end{equation}
The dequantized weights are then reconstructed as:
\begin{equation}
W_{\text{deq},i,j} = \tilde{W}_{i,j}\,\delta_i + \mu_i.
\end{equation}
Centering enforces a zero-mean distribution along each row, stabilizing the binarization process. 
However, prior studies indicate that LLM weights typically exhibit unimodal, quasi-Gaussian distributions within each channel. Direct binarization of such unimodal distributions inevitably induces substantial quantization error, severely degrading model performance. 
In contrast, a \textit{bimodal} distribution aligns naturally with the binary codebook ($\{-1, +1\}$), forming two distinct clusters that theoretically minimize quantization error (see Appendix~\ref{app:double-bell-binary}). 
Despite this theoretical advantage, practical LLM weights remain predominantly unimodal. Although QAT can implicitly induce bimodal distributions, applying QAT to LLMs typically incurs prohibitively high training costs, making it impractical in real-world settings.
To address these limitations, we propose \textbf{BWLA} (illustrated in Figure~\ref{fig:overview}), an efficient PTQ framework that explicitly transforms unimodal weights into bimodal forms while suppressing activation outliers—achieving both objectives without expensive retraining. The complete pseudocode is provided in Appendix~\ref{app:pseudocode}, and distribution visualizations are available in Appendix~\ref{app:distribution}.

\subsection{Orthogonal--Kronecker Transformation}
\label{sec:OKT}

\paragraph{Theoretical Feasibility via Orthogonal Bimodalization.}
Reshaping pre-trained unimodal weights into a bimodal distribution is crucial for joint quantization. While prior work~\citep{ye2025dbellquantbreakingbelldoublebell} explored arbitrary auxiliary matrices, guaranteeing their strict invertibility without destabilizing the forward pass is problematic. We circumvent this by utilizing the orthogonal group $\mathcal{O}(m)$, which guarantees cost-free inversion ($R^{-1} = R^\top$). We establish that strict orthogonality theoretically induces the desired distributional shift:

\begin{theorem}
\label{theo1}
Let $W\in\mathbb{R}^{n\times m}$ be a weight matrix with rows independently sampled from a unimodal Gaussian $w_i \sim \mathcal N(\mathbf{0},\sigma_i^{2}I_m)$~\citep{ye2025dbellquantbreakingbelldoublebell}. There exists an orthogonal transformation $R\in\mathcal O(m)$ such that the transformed weights $W' = WR$ converge to a symmetric bimodal Gaussian mixture:
\begin{equation}
w'_i \sim \pi\,\mathcal N(\mu,\sigma^{2}) + (1-\pi)\,\mathcal N(-\mu,\sigma^{2}),
\label{eq:gaussian}
\end{equation}
where $\pi \in (0,1)$ is the mixture coefficient, and $\mu, \sigma > 0$. \textit{Proof.} See Appendix~\ref{app:proof-aux}.
\end{theorem}

While Theorem~\ref{theo1} proves theoretical feasibility, deploying a dense $m\times m$ orthogonal matrix in LLMs incurs prohibitive $\mathcal{O}(m^2)$ memory and $\mathcal{O}(nm^2)$ computational overhead. Furthermore, an unconstrained $R$ is agnostic to activation statistics, lacking explicit mechanisms to suppress heavy-tailed outliers or ensure numerical stability under extreme low-bit constraints. These limitations necessitate a shift towards a highly structured, hardware-aware transformation.

\paragraph{Orthogonal-Kronecker Transformation (OKT).} 
To enable lightweight yet strictly orthogonal mixing, we propose the Orthogonal-Kronecker Transformation (OKT). A full orthogonal matrix $R \in \mathbb{R}^{m \times m}$ incurs prohibitive $\mathcal{O}(m^2)$ computational and memory overhead. To circumvent this, OKT imposes a Kronecker structure by factorizing the transformation as $R = R_1 \otimes R_2$, where $R_1 \in \mathbb{R}^{n_1 \times n_1}$ and $R_2 \in \mathbb{R}^{n_2 \times n_2}$ ($n_1 n_2 = m$).

By reshaping a weight row $v \in \mathbb{R}^{1 \times m}$ into a matrix $V_{\text{mat}} \in \mathbb{R}^{n_1 \times n_2}$, the transformation can be efficiently evaluated via $R_2^{\top} V_{\text{mat}} R_1$~\citep{xiao2025singlequantefficientquantizationlarge}. When $n_1$ and $n_2$ are optimally chosen to be approximately $\sqrt{m}$, this structural constraint drastically cuts the computational complexity down to $\mathcal{O}(m^{3/2})$ and reduces the memory footprint by orders of magnitude (e.g., storing two $64 \times 64$ matrices instead of one $4096 \times 4096$ matrix). Beyond computational efficiency, the strict orthogonality of OKT is crucial: it exactly preserves forward-pass invariance while effectively dispersing activation coherence. Consequently, OKT provides a unified mechanism to naturally suppress heavy-tailed outliers and facilitate subsequent weight bimodalization with minimal overhead.

\paragraph{Transformation Objectives.}
To explicitly induce bimodality, we first center each transformed row:
\begin{equation}
x_i = (w_i R)-\tfrac{1}{m}\mathbf{1}\mathbf{1}^\top (w_i R), 
\end{equation}
where $\mathbf{1}\in\mathbb{R}^{m}$ is the all-ones vector, such that the resulting entries $x_{ij}$ are symmetric around zero. Each row is modeled by a symmetric two-component mixture distribution:
\begin{equation}
\begin{split}
p(x_{ij}&\mid c_i,\sigma_i^2)
= 
\\
&\tfrac12\,\phi(x_{ij};+c_i,\sigma_i^2)
+ \tfrac12\,\phi(x_{ij};-c_i,\sigma_i^2),
\end{split}
\end{equation}
where $p(\cdot)$ denotes the probability density function, 
$\phi(\cdot;\mu,\sigma^2)$ is a Gaussian density with mean $\mu$ and variance $\sigma^2$, 
and $\Theta_i = \{c_i, \sigma_i^2\}$ represents the distribution parameters for the $i$-th row. The negative log-likelihood objective is:
\begin{equation}
\begin{split}
L&_{\mathrm{GMM}}
= -\frac{1}{nm} \sum_{i,j}
\log \\
&\left[ 
\tfrac12\phi(x_{ij};c_i,\sigma_i^2)
+ \tfrac12\phi(x_{ij};-c_i,\sigma_i^2)
\right].
\end{split}
\end{equation}
To prevent degenerate solutions (e.g., $c_i \to 0$ or mode collapse), we introduce a regularization term on the posterior responsibilities of the positive component.
Specifically, let $r_{ij}^+$ denote the posterior probability that entry $x_{ij}$ is assigned to the positive mode, and define the row-wise average $\bar r_i = \tfrac{1}{m}\sum_j r_{ij}^+$ .
By penalizing the deviation of $\bar r_i$ from $1/2$, we encourage balanced assignments between the two modes. The resulting optimization objective is
\begin{equation}
\min_{R\in\mathcal O(m),\,\Theta}\;
L_{\mathrm{GMM}}
+ \lambda_{\text{reg}}\left(\tfrac1n\sum_{i=1}^{n}\bar r_i-\tfrac12\right)^2 ,
\end{equation}

We employ an alternating optimization strategy. For fixed $R$, $\Theta_i$ is updated via closed-form EM steps~\citep{dempster1977maximum}. For fixed $\Theta_i$, $R_1$ and $R_2$ are updated via Majorization-Minimization (MM)~\citep{hunter2004tutorial}, which decomposes into weighted orthogonal Procrustes problems. Specifically, $R_1$ is updated as
\begin{equation}
\begin{split}
C_1 = \sum_{i=1}^{n} \lambda_i &V_{\mathrm{mat},i}^{\top} R_2 M_i,\quad R_1^{\mathrm{new}} = UV^\top, \\
\quad &\text{where } U\Sigma V^\top=\mathrm{SVD}(C_1),
\end{split}
\end{equation}
The update for $R_2$ follows analogously. Detailed derivations and convergence guarantees are provided in Appendix~\ref{app:okt-em-mm}.


\input{tab/tab_1}

\subsection{Proximal SVD Projection}
\label{sec:PSP}
While OKT aligns weights with a bimodal geometry globally, orthogonal rotations alone cannot eliminate structured residuals that remain misaligned with the target centers $\pm c_i$. To address this, we introduce \textbf{Proximal SVD Projection (PSP)}, a lightweight refinement step.

We parameterize a rank-$k$ residual matrix $M=AB$ ($k \ll \min\{\mathrm{oc},\mathrm{ic}\}$). The corrected weight is $W_{\mathrm{res}}=W-M$. We define the OKT-centered projection as:
\begin{equation}
T_R(W)= (W R)\,H, \quad \text{with } H = I - \tfrac1m \mathbf{1}\mathbf{1}^\top.
\end{equation}
The transformed weights $D = T_R(W-M)$ are optimized under the same GMM objective. Since $M$ acts as an affine shift, it does not alter the closed-form nature of the EM updates for $\Theta$, ensuring:
\begin{equation}
L_{\mathrm{GMM}}(\Theta^{(t+1)}; \dots) \le L_{\mathrm{GMM}}(\Theta^{(t)}; \dots).
\end{equation}

To update $M$, we treat it as the primary variable with a constraint $\operatorname{rank}(M)\le k$. Let $G_D = \partial L/\partial D$ be the gradient in the transformed space. The adjoint operator maps this gradient back to the residual domain:
\begin{equation}
G = T_R^{*}(G_D) = G_D H^\top R^\top = G_D H R^\top.
\end{equation}
We construct a proximal majorizer $Q_{\mu}$ at step $t$:
\begin{equation}
\resizebox{\hsize}{!}{$
Q_{\mu}(M|M_t) = L(M_t) + \langle G, M-M_t \rangle + \frac{\mu}{2}\|M-M_t\|_F^2.
$}
\end{equation}
Minimizing $Q_\mu$ implies solving the proximal operator:
\begin{equation}
M_{t+1} = \arg\min_{\operatorname{rank}(M)\le k} \left\| M - (M_t - \tfrac{1}{\mu} G) \right\|_F^2.
\end{equation}
Defining $Y_t = M_t - \tfrac{1}{\mu} G$, the optimal solution is given by the truncated SVD:
\begin{equation}
Y_t \approx U_k\Sigma_kV_k^\top \implies M_{t+1} = U_k\Sigma_k V_k^\top.
\end{equation}
This update guarantees monotonic decrease in the objective $L(W-M_{t+1}) \le L(W-M_t)$ without manual step-size tuning. PSP effectively absorbs residual outliers at negligible parameter cost. Full derivations are in Appendix~\ref{app:psp-details}.

%% file: tab/tab_1.tex
\begin{table*}[h!]
\centering
\resizebox{\textwidth}{!}{%
\tablestyle{2pt}{1.2}
\begin{tabular}{l|c|c|*{2}{cc:}cc|cc|cc:cc:cc}
\toprule
\multirow{3}{*}{\textbf{Method}} & \multirow{3}{*}{\makecell{\textbf{\#Bits}\\ (W)}} & \multirow{3}{*}{\makecell{\textbf{\#Bits}\\ (A)}}    &
\multicolumn{2}{c}{\textbf{LLaMA2-7B}} &
\multicolumn{2}{c}{\textbf{LLaMA2-13B}} &
\multicolumn{2}{c}{\textbf{LLaMA2-70B}} &
\multicolumn{2}{c}{\textbf{LLaMA3-8B}} &
\multicolumn{2}{c}{\textbf{Qwen3-8B}} &
\multicolumn{2}{c}{\textbf{Qwen3-14B}} &
\multicolumn{2}{c}{\textbf{Qwen3-32B}} \\
\cdashline{4-17}
 & & & 0-shot$^5$ & Wiki & 0-shot$^5$ & Wiki & 0-shot$^5$ & Wiki & 0-shot$^5$ & Wiki & 0-shot$^5$ & Wiki & 0-shot$^5$ & Wiki & 0-shot$^5$ & Wiki \\
 & & & Avg.($\uparrow$) &($\downarrow$) & Avg.($\uparrow$) &($\downarrow$) & Avg.($\uparrow$) &($\downarrow$) & Avg.($\uparrow$) &($\downarrow$) & Avg.($\uparrow$) &($\downarrow$) & Avg.($\uparrow$) &($\downarrow$) & Avg.($\uparrow$) &($\downarrow$)\\
\hdashline
FP16       & 16   & 16 & 68.96 & 5.47 & 71.71 & 4.88 & 76.55 & 3.32 & 72.67 & 6.14 & 71.47 & 9.00 & 75.00 & 8.64 & 76.33 & 7.61 \\
\hdashline
RTN        & 2    & \multirow{7}{*}{16} & 35.76 & 6e4 & 36.20 & 9e3 & 35.87 & 1e4 & 36.15 & 2e6 & 35.39 & 7e6 & 35.77 & 1e7 & 35.57 & 1e7 \\
GPTQ       & 2    &  & 36.06 & 5e4 & 35.68 & 9e3 & 35.96 & 2e4 & 35.92 & 8e5 & 35.27 & 8e3 & 35.72 & 1e5 & 35.72 & 1e5 \\
OSTQuant   & 2    &  & 37.71 & 3e2 & 39.58 & 8e2 & 42.54 & 2e2 & 37.58 & 9e2 & 37.92 & 2e2 & 41.08 & 3e2 & 44.36 & 4e2 \\
BiLLM      & 1.07 &  & 40.91 & 19.40 & 44.38 & 13.39 & 58.76 & 8.83 & 38.13 & 37.66 & 41.58 & 44.04 & 54.34 & 18.07 & 60.26 & 13.23 \\
ARB-LLM    & 1.07 &  & 43.22 & 17.06 & 52.04 & 10.75 & 63.36 & 7.03 & 43.80 & 22.37 & 47.67 & 21.50 & 55.79 & 15.05 & 66.51 & 11.67 \\
DBellQuant & 1.10 &  & 44.36    & 17.91 & --    & 12.79 & --    & 6.84   & 43.12  & -- & --    & --   & --    & --    & -- & --    \\
\rowcolor{gray!10}
\textbf{BWLA} & 1.16 &  & \textbf{52.46} & \textbf{9.96} & \textbf{62.38} & \textbf{7.12} & \textbf{73.78} & \textbf{5.03} & \textbf{48.81} & \textbf{15.21} & \textbf{52.11} & \textbf{16.29} & \textbf{62.65} & \textbf{12.66} & \textbf{68.29} & \textbf{10.72} \\
\hdashline
RTN        & 2    & \multirow{7}{*}{6}  & 36.20    & 6e4  & 36.58    & 1e4 & 36.35    & 1e4 & 35.51    & 2e6 & 36.12    & 7e6 & 36.02    & 3e6 & 35.55    & 5e4 \\
GPTQ       & 2    &   & 36.16    & nan   & 35.95    & 8e3 & 35.45    & nan  & 34.92    & 1e6 & 34.38    & 1e7 & 34.38    & 6e6 & 34.36    & 4e5 \\
OSTQuant   & 2    &   & 37.80    & 9e2  & 37.76    & 6e2 & 36.89    & 3e2 & 36.45    & 3e2 & 34.75    & 3e2 & 36.77    & 8e3 & 37.29    & 6e3 \\
BiLLM      & 1.07 &   & 38.86    & 38.00 & 40.84   & 33.24 & 37.93   & 26.72 & 38.36   & 59.77 & 35.89   & 434.86 & 42.25   & 765.25 & 45.95   & 485.05 \\
ARB-LLM    & 1.07 &   & 40.45    & 24.98 & 45.44   & 14.12 & 44.49   & 19.47 & 41.22   & 28.24 & 36.48   & 215.09 & 40.56   & 1e4   & 38.86   & 38.00 \\
DBellQuant & 1.10 &   & 41.32    & 21.69 & --   & 14.39 & --   & 7.56 & 39.64 & -- & --   & --    & --   & --    & --   & --    \\
\rowcolor{gray!10}
\textbf{BWLA} & 1.16 &   & \textbf{45.90} & \textbf{12.19} & \textbf{60.07} & \textbf{7.60} & \textbf{72.38} & \textbf{5.35} & \textbf{45.79} & \textbf{17.94} & \textbf{50.46} & \textbf{17.78} & \textbf{60.07} & \textbf{13.80} & \textbf{67.17} & \textbf{11.92} \\
\bottomrule
\end{tabular}}
\caption{Comparison of perplexity on WikiText2 and averaged accuracy on five Zero-Shot tasks(Arc-Challenge, Arc-Easy , HellaSwag , PIQA, and WinoGrande). Full results are in the Appendix~\ref{app:detailed main results}}
\label{tab:main_results}
\end{table*}

%% file: sec/4_experiments.tex
\section{Experiments}
\label{sec:Experiments}

\subsection{Experiment setup}
\paragraph{Models and Datasets.}
We conduct experiments on the LLaMA families~\citep{touvron2023llamaopenefficientfoundation} and the Qwen3 family~\citep{qwen3technicalreport}. In addition, we evaluate instruction-tuned variants to further demonstrate the effectiveness of our method. 
Beyond standard perplexity evaluation on Wikitext2~\citep{wikitext2} and C4~\citep{c4}, we assess BWLA on a broad set of zero-shot tasks, including ARC-Challenge and ARC-Easy~\citep{clark2018thinksolvedquestionanswering}, HellaSwag~\citep{zellers2019hellaswagmachinereallyfinish}, LAMBADA-openai and LAMBADA-standard~\citep{paperno2016lambadadatasetwordprediction}, PIQA~\citep{bisk2019piqareasoningphysicalcommonsense}, and WinoGrande~\citep{sakaguchi2019winograndeadversarialwinogradschema}. 
We further evaluate BWLA on more challenging reasoning benchmarks, including the multi-domain knowledge task MMLU~\citep{mmlu}, the mathematical reasoning benchmark GSM8K~\citep{gsm8k}, and the code generation benchmark HumanEval~\citep{humaneval}.

\input{tab/tab_2}
\input{tab/tab_3}
\input{tab/tab_4}

\paragraph{Baseline Methods.} 
We primarily compare BWLA against BiLLM~\citep{huang2024billmpushinglimitposttraining}, ARB-LLM~\citep{li2024arbllmalternatingrefinedbinarizations}, and DBellQuant~\citep{ye2025dbellquantbreakingbelldoublebell}, where DBellQuant represents the current SOTA PTQ approach applied to binary LLMs. We also include several recent PTQ baselines, such as RTN (round-to-nearest), GPTQ~\citep{frantar2023gptqaccurateposttrainingquantization}, and OSTQuant~\citep{hu2025ostquantrefininglargelanguage}, which are among the strongest low-bit quantization methods.

\paragraph{Implementation Details.} 
All experiments are conducted on NVIDIA A6000 GPUs. We apply per-token asymmetric quantization for activations, while weights are quantized using symmetric per-channel quantization. For calibration, 128 text segments are sampled from the Wikitext2. We set the total number of BWLA iterations to 60, with 40 for OKT and 20 for PSP (corresponding loss curves are shown in Appendix~\ref{app:loss}). As BWLA is an efficient PTQ framework, no fine-tuning is required. For the three complex reasoning tasks, MMLU, GSM8k, HumanEval, and several other zero-shot tasks, we use the open-source tool lm-evaluation-harness~\citep{gao2024lmeval} for assessment.

\subsection{Main Results}

\paragraph{Comparison Results.}
We systematically evaluate binary performance across multiple LLM families under varying activation bit-widths. For fair comparison, non-binarization PTQ baselines (RTN, GPTQ, and OSTQuant) use 2-bit weights. Table~\ref{tab:main_results} shows that BWLA consistently outperforms all competitors on both LLaMA and Qwen3. Under weight-only binarization (A16), conventional low-bit methods collapse even at 2-bit weights, underscoring their limitations for extremely low-bit LLMs.
Compared with current state-of-the-art binary methods, BWLA improves average accuracy by 13\% and reduces perplexity by 28\%. When activation precision is reduced to 6 bits (A6), existing binary approaches degrade sharply: BWLA achieves up to 37\% lower perplexity than DBellQuant on LLaMA and, more importantly, over 50\% average accuracy gains where BiLLM and ARB-LLM nearly collapse on the Qwen3 family. This stability, achieved with less than 0.1 bit increase in effective weight precision, demonstrates BWLA’s robustness and efficiency as a practical solution for high-accuracy W1AX quantization. Detailed results are provided in Appendix~\ref{app:detailed main results}.

\begin{figure*}[t!]
    \centering
    \includegraphics[width=0.75\linewidth]{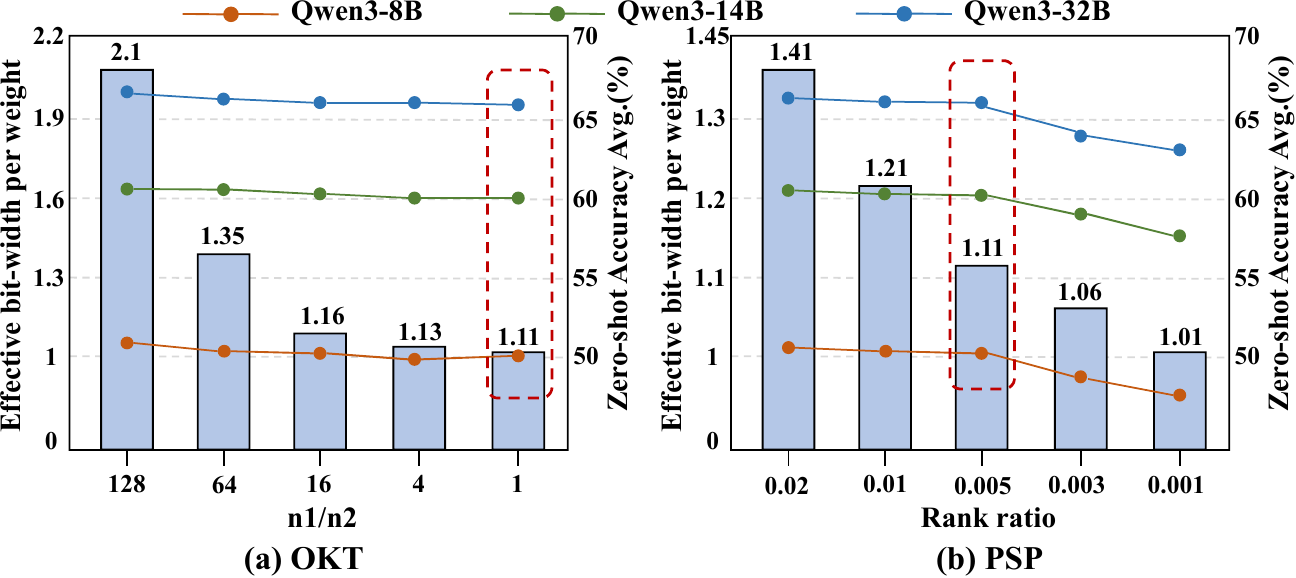} 
    \caption{Ablation of the Overhead–Performance trade-off for OKT and PSP. The results show that the optimal balance is achieved when OKT uses $n_1/n_2$ = \textcolor[HTML]{8B0000}{1} and PSP adopts a rank ratio of \textcolor[HTML]{8B0000}{0.005}.}
    \label{fig:overhead}
\end{figure*}

\paragraph{Experiments of Instruction-tuned Models.}
Instruction tuning greatly improves the practical utility of LLMs but also makes quantization more challenging than for base models. We evaluate Qwen3-32B-Instruct on three reasoning benchmarks (Table~\ref{tab:qwen32b_instruct}). With FP16 activations, BWLA retains about 75\% of full-precision performance, recovering most reasoning capability and even outperforming GPTQ with 3-bit weights while using less than half of its memory footprint.
When activations are further quantized to 6 bits, prior methods nearly collapse: MMLU accuracy approaches random-guess levels (25\%), and both HumanEval and GSM8K drop to zero. In contrast, BWLA preserves approximately 94\% of its performance relative to the unquantized-activation setting, demonstrating strong quantization quality under joint weight–activation compression.

\subsection{Ablation Studies}
BWLA integrates two key components, OKT and PSP, each designed to reduce quantization error in LLMs. Our ablation study is organized into three parts: (i) evaluating performance under aggressive 4-bit activation quantization, (ii) assessing the individual contribution of each module, and (iii) analyzing the Overhead–Performance Trade-off within the OKT and PSP components. The ablation studies on the impact of calibration data are presented in Appendix~\ref{app:dataset}.

\paragraph{Lower Activation Bitwidth.}
Under a more aggressive 4-bit activation quantization setting, we evaluate the generalizability of our method on LLaMA2-7B, LLaMA3-8B, and Qwen3-14B. As shown in Table~\ref{tab:w1a4}, GPTQ, BiLLM, and ARB-LLM all collapse across the seven zero-shot tasks, with perplexities commonly exceeding the 1e4 range. In contrast, BWLA delivers over a 50\% average improvement in zero-shot accuracy and reduces perplexity by more than 99.9\%, achieving an average value of only 43.26, demonstrating remarkable robustness and strong overall performance.

\paragraph{Modular Sensitivity Study.}
We analyze the individual and combined effects of OKT and PSP by evaluating average perplexity on WikiText2 and C4, together with zero-shot accuracy on seven benchmarks: Arc-Challenge, Arc-Easy, HellaSwag, LAMBADA-openai, LAMBADA-standard, PIQA, and WinoGrande.
As shown in Table~\ref{tab:method_ablation}, (i) each module alone improves both perplexity and zero-shot accuracy, with OKT yielding the larger gains, and (ii) combining the two modules produces the best results. Notably, on Qwen3-32B with 6-bit activation quantization, our joint design reduces average perplexity by 99.9\% and improves zero-shot accuracy by over 90\% compared to the baseline, highlighting the strong synergy and necessity of OKT and PSP.

\paragraph{Overhead–Performance Trade-off.}
We ablate the orthogonal matrix size in OKT and the truncated SVD rank ratio in PSP to examine the overhead–performance trade-off. As shown in Fig.~\ref{fig:overhead}(a), fixing the PSP rank ratio to 0.005 and varying the Kronecker dimensions $n_1$ and $n_2$ shows that decreasing $n_1/n_2$ (i.e., increasing symmetry) substantially reduces the effective weight bit-width introduced by OKT with minimal performance loss. On the Qwen3 family, greater symmetry lowers OKT overhead from about 1 bit to a negligible 0.01 bit, with only a $\sim$3\% drop in average zero-shot accuracy.
Figure~\ref{fig:overhead}(b) reports the PSP ablation. Reducing the rank ratio from 0.02 to 0.005 causes only a $\sim$1\% performance loss while saving about 0.3 bits of overhead, whereas further reduction leads to rapid degradation. On Qwen3-8B, decreasing the ratio from 0.005 to 0.001 results in over a 10\% accuracy drop while saving merely $\sim$0.1 bit, with similar trends across models. Overall, a rank ratio of 0.005 achieves the best overhead–performance balance and is the preferred PSP setting.

\begin{figure*}[t!]
    \centering
    \includegraphics[width=0.9\linewidth]{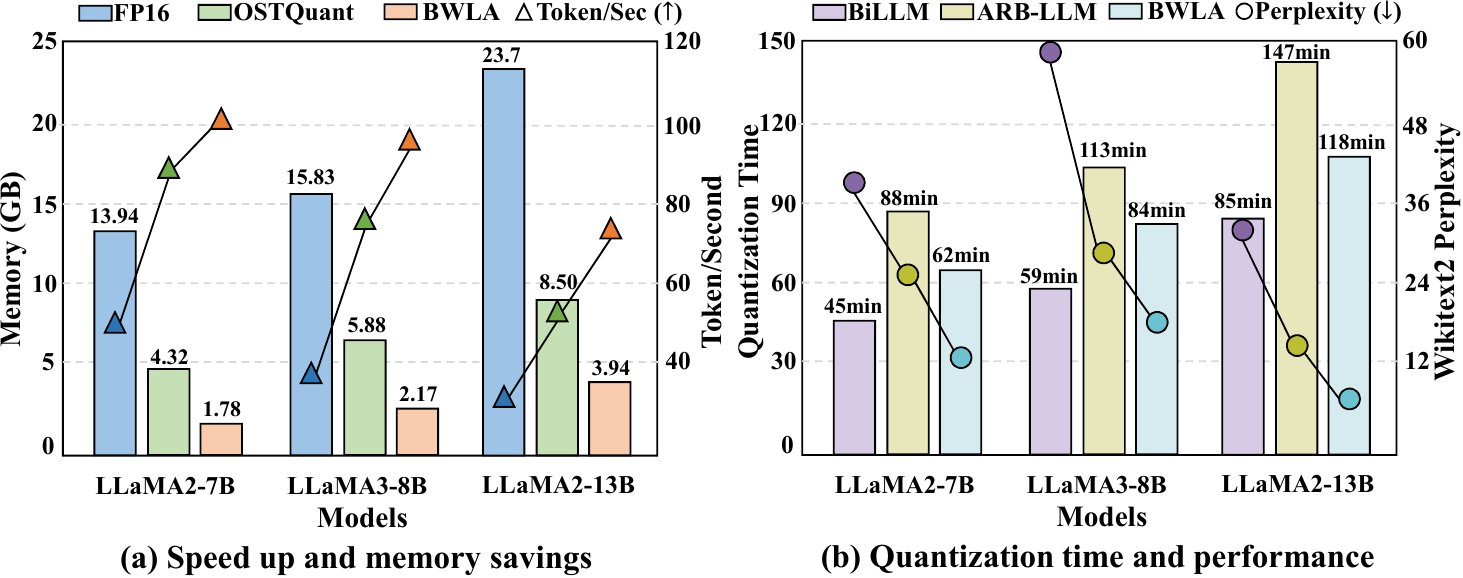}
    \caption{Efficiency Analysis. (a) Comparison of throughput (Tokens/Sec) and memory consumption across FP16, OSTQuant, and BWLA. (b) Comparison of quantization time and perplexity across BiLLM, ARB-LLM, and BWLA.}
    \label{fig:efficiency}
\end{figure*}

\subsection{Efficiency Analysis of BWLA}

\paragraph{Speedup and Memory Savings.} 
We benchmark the efficiency of BWLA on NVIDIA RTX A6000 GPUs using LLaMA families. We compare FP16, a W4A4 model quantized by OSTQuant, and a W1A8 model quantized by BWLA in terms of throughput (tokens/s) and memory usage. Under a batch size of 4, a 1024-token prefill, and 256-token decoding, Fig.~\ref{fig:efficiency}(a) shows that decoding is typically memory-bound; thus, reducing weight precision to 1 bit substantially lowers memory traffic compared to 4-bit quantization, leading to faster inference. On LLaMA2-13B, OSTQuant increases throughput from 23.70 tokens/s (FP16) to 55.35 tokens/s, while BWLA boosts it to 77.31 tokens/s—achieving a 3.26$\times$ speedup over FP16 and a 1.40$\times$ speedup over OSTQuant. Meanwhile, parameter memory is reduced from 23.7 GB to 3.94 GB, saving over 80\% of storage. These results highlight the significant efficiency gains enabled by BWLA.

\paragraph{Quantization Time Comparison.}
As a fine-tuning-free PTQ framework, BWLA applies OKT and PSP to each weight matrix entirely offline, converging in only a few dozen iterations (see Appendix~\ref{app:loss}) with no runtime overhead. Figure~\ref{fig:efficiency}(b) shows that BWLA’s quantization time is substantially lower than ARB-LLM, whose column-group bitmaps rely on costly percentile-based searches. Compared with BiLLM, which lacks iterative updates, BWLA is only about 25 minutes slower on average across the three evaluated models.
Despite this modest overhead, under W1A6 quantization BWLA reduces WikiText2 perplexity by over 70\% relative to BiLLM and by more than 50\% compared to the slower ARB-LLM, demonstrating strong advantages in both quantization efficiency and final accuracy.

%% file: tab/tab_2.tex
\begin{table}[t]
\centering
\resizebox{0.48\textwidth}{!}{%
\tablestyle{2pt}{1.2}
\begin{tabular}{clccccc}
\toprule
\textbf{Model} & \textbf{Method} & \textbf{\#Bits(W)} & \textbf{\#Bits(A)} & \textbf{MMLU} & \textbf{HumanEval} & \textbf{GSM8K} \\
\midrule
\multirow{11}{*}{\makecell{Qwen3-32B \\ -Instruct}}
& FP16    & 16   & 16 & 80.75 & 38.41 & 62.55 \\
\cline{2-7}
& GPTQ    & 3    & \multirow{5}{*}{16} & 63.75 & 23.40 & 34.15 \\
& GPTQ    & 2    &  & 24.58 &  0 &  0 \\
& BiLLM   & 1.06 &  & 49.04 & 10.15 & 12.12 \\
& ARB-LLM & 1.06 &  & 62.56 & 18.32 & 23.56 \\
& \cellcolor{gray!10}\textbf{BWLA}    & \cellcolor{gray!10}1.15 &\cellcolor{gray!10}  &\cellcolor{gray!10} \textbf{67.74} &\cellcolor{gray!10} \textbf{25.31} &\cellcolor{gray!10} \textbf{42.28} \\
\cline{2-7}
& GPTQ    & 3    &  \multirow{5}{*}{6} & 25.36 &  0 &  0 \\
& GPTQ    & 2    &   & 24.75 &  0 &  0 \\
& BiLLM   & 1.06 &   & 25.82 &  0 &  0 \\
& ARB-LLM & 1.06 &   & 24.33 &  0 &  0 \\
& \cellcolor{gray!10}\textbf{BWLA}    & \cellcolor{gray!10}1.15 &\cellcolor{gray!10}   & \cellcolor{gray!10}\textbf{67.48} & \cellcolor{gray!10}\textbf{23.09} & \cellcolor{gray!10}\textbf{38.16} \\
\bottomrule
\end{tabular}}
\caption{Results on the instruction-tuned model Qwen3-32B-Instruct across three challenging reasoning benchmarks: MMLU, HumanEval, and GSM8K.}
\label{tab:qwen32b_instruct}
\end{table}

%% file: tab/tab_3.tex
\begin{table*}[h!]
\centering
\resizebox{0.85\linewidth}{!}{
\tablestyle{3pt}{1.2}
\begin{tabular}{clcccccccccc}
\toprule
\textbf{Model} & \textbf{Method} & \textbf{\#Bits(A)} & \textbf{AE}$\uparrow$ & \textbf{AC}$\uparrow$ & \textbf{HS}$\uparrow$ & \textbf{LO}$\uparrow$ & \textbf{LS}$\uparrow$ & \textbf{PQ}$\uparrow$ & \textbf{WG}$\uparrow$ & \textbf{$\text{Avg.}^7$}$\uparrow$ & \textbf{Wiki PPL}$\downarrow$ \\
\midrule
\multirow{4}{*}{LLaMA2-7B}
& GPTQ   & \multirow{4}{*}{4}
& 26.05 & 28.07 & 25.43 & 0    & 0    & 50.92 & 50.91 & 25.91 & 1e5 \\
& BiLLM  &
& 26.01 & 25.68 & 25.27 & 0    & 0    & 48.97 & 50.51 & 25.21 & 6e6 \\
& ARB-LLM    &
& 27.10 & 25.77 & 26.11 & 0    & 0    & 50.60 & 47.83 & 25.34 & nan \\
\rowcolor{gray!10}
& \textbf{BWLA} &
& \textbf{39.21} & \textbf{27.12} & \textbf{39.31} &
\textbf{15.30} & \textbf{16.40} &
\textbf{55.72} & \textbf{52.09} &
\textbf{35.00} & \textbf{39.88} \\
\midrule
\multirow{4}{*}{LLaMA3-8B}
& GPTQ   & \multirow{4}{*}{4}
& 24.45 & 26.45 & 25.97 & 0    & 0    & 51.41 & 50.43 & 25.53 & 6e5 \\
& BiLLM  &
& 26.43 & 25.60 & 26.89 & 0.04 & 0.12 & 50.16 & 49.49 & 25.53 & 6e3 \\
& ARB-LLM    &
& 28.45 & 26.62 & 25.66 & 0.27 & 0.08 & 50.82 & 48.38 & 25.75 & 9e3 \\
\rowcolor{gray!10}
& \textbf{BWLA} &
& \textbf{43.54} & \textbf{27.93} & \textbf{40.27} &
\textbf{17.57} & \textbf{17.82} &
\textbf{56.42} & \textbf{52.28} &
\textbf{36.55} & \textbf{55.12} \\
\midrule
\multirow{4}{*}{Qwen3-14B}
& GPTQ   & \multirow{4}{*}{4}
& 25.08 & 22.70 & 25.04 & 0    & 0    & 49.51 & 49.57 & 24.56 & 9e5 \\
& BiLLM  &
& 24.07 & 26.88 & 25.64 & 0    & 0    & 51.03 & 51.07 & 25.53 & 2e5 \\
& ARB-LLM    &
& 25.67 & 27.05 & 26.35 & 0    & 0    & 50.60 & 53.59 & 26.18 & 2e4 \\
\rowcolor{gray!10}
& \textbf{BWLA} &
& \textbf{44.61} & \textbf{29.95} & \textbf{44.99} &
\textbf{29.73} & \textbf{23.73} &
\textbf{61.97} & \textbf{53.83} &
\textbf{41.26} & \textbf{34.77} \\
\bottomrule
\end{tabular}}
\caption{Perplexity on WikiText2 and zero-shot accuracy on Arc-Challenge (AC), Arc-Easy (AE), HellaSwag (HS), LAMBADA-openai (LO), LAMBADA-standard (LS), PIQA (PQ), and WinoGrande (WG) under a \textbf{4-bit activation} quantization setting. Since GPTQ is not designed for binarization, we evaluate it using 3-bit weight quantization.}
\label{tab:w1a4}
\end{table*}

%% file: tab/tab_4.tex
\begin{table}[t]
\centering
\resizebox{\linewidth}{!}{
\tablestyle{2.5pt}{1.2}
\begin{tabular}{c |c:c| *{2}{cc|} cc}
\toprule
\multirow{3}{*}{\makecell{\textbf{\#Bits} \\(A)}} & \multirow{3}{*}{\textbf{OKT}} & \multirow{3}{*}{\textbf{PSP}} &
\multicolumn{2}{c}{\textbf{LLaMA2-13B}} &
\multicolumn{2}{c}{\textbf{LLaMA3-8B}} &
\multicolumn{2}{c}{\textbf{Qwen3-32B}} \\

& & & PPL & 0-shot$^{7}$ & PPL & 0-shot$^{7}$ & PPL & 0-shot$^{7}$ \\
&& &Avg.($\downarrow$) &Avg.($\uparrow$) &Avg.($\downarrow$) &Avg.($\uparrow$) &Avg.($\downarrow$) &Avg.($\uparrow$) \\
\midrule
\multirow{4}{*}{16}
 & $\times$       & $\times$       & 21.66 & 40.10 & 182.32 & 30.56 & 20.97 & 58.17 \\
 & \checkmark     & $\times$       & 10.12 & 60.04 &  28.80 & 45.73 & 14.67 & 64.58 \\
 & $\times$       & \checkmark     & 16.64 & 49.52 &  40.88 & 38.94 & 17.70 & 52.45 \\
 \rowcolor{gray!10}
 & \checkmark     & \checkmark     & \textbf{9.37} & \textbf{62.38} &
                                  \textbf{24.49} & \textbf{48.81} &
                                  \textbf{13.91} & \textbf{68.29} \\
\midrule
\multirow{4}{*}{6}
 & $\times$       & $\times$       & 48.75 & 34.71 & 213.62 & 30.74 & 4e4   & 35.05 \\
 & \checkmark     & $\times$       & 11.53 & 57.90 &  36.25 & 45.25 & 15.65 & 51.21 \\
 & $\times$       & \checkmark     & 20.80 & 45.73 &  43.59 & 37.23 & 19.77 & 47.19 \\
 \rowcolor{gray!10}
 & \checkmark     & \checkmark     & \textbf{10.41} & \textbf{60.07} &
                                  \textbf{31.77} & \textbf{45.79} &
                                  \textbf{14.33} & \textbf{67.17} \\
\bottomrule
\end{tabular}}
\caption{Impact of different components in BWLA.}
\label{tab:method_ablation}
\end{table}

%% file: sec/5_conclusion.tex
\section{Conclusion}
In this work, we identified two key obstacles to accurate W1AX PTQ: the mismatch between unimodal weights and binary codebooks, and heavy-tailed activation outliers. To overcome these, we proposed BWLA, a unified framework combining the Orthogonal-Kronecker Transformation (OKT) for bimodal reshaping and activation smoothing, and the Proximal SVD Projection (PSP) for lightweight low-rank refinement. Extensive evaluations across multiple LLM families show that BWLA delivers the first practical, high-accuracy PTQ solution for binary weights with low-bit activations. By significantly outperforming prior methods in accuracy and robustness, BWLA enables substantial gains in throughput and memory efficiency, paving the way for effective, end-to-end accelerated lightweight LLM inference.

%% file: appendix/A.1.tex
\subsection{Pseudocode of BWLA}
\label{app:pseudocode}
For completeness, this section provides the detailed pseudocode of our 
\textbf{BWLA} framework. As introduced in Section~\ref{sec:OKT}, we construct the large 
auxiliary transformation matrix $R \in \mathbb{R}^{n\times n}$ using two 
lightweight Kronecker factors $R_1 \in \mathbb{R}^{n_1\times n_1}$ and 
$R_2 \in \mathbb{R}^{n_2\times n_2}$. To obtain a suitable Kronecker 
decomposition, we first apply Algorithm~\ref{alg:kron_factor} to factorize the hidden dimension $n$ 
into a pair $(n_1,n_2)$, which enables efficient construction of the orthogonal 
Kronecker structure.

Algorithm~\ref{alg:okt_psp} presents the complete \textbf{BWLA} pipeline, which consists of two 
major components: (1) \emph{OKT (Orthogonal–Kronecker Transformation)}, which 
rotates the weight matrix into a coordinate system that naturally exposes a 
symmetric bimodal structure; and (2) \emph{PSP (Proximal SVD Projection)}, which 
removes the structured residual energy that cannot be eliminated through 
orthogonal mixing alone.

The pseudocode summarizes the full EM–MM optimization procedure adopted by BWLA, 
including responsibility computation, mixture-parameter updates, alternating 
orthogonal Procrustes steps for the Kronecker factors, and the proximal low-rank 
refinement. For clarity, the OKT and PSP modules are highlighted with color-coded 
comments in the algorithm. The resulting pseudocode faithfully reflects the 
transformation pipeline used in all our experiments and can be directly applied 
to large language models in post-training quantization (PTQ) scenarios.

\begin{algorithm}[h]
\caption{Kronecker Dimension Factorization for OKT}
\label{alg:kron_factor}
\DontPrintSemicolon
\SetKwInOut{Input}{Input}
\SetKwInOut{Output}{Output}

\Input{Hidden dimension $n$ of the linear layer ($n \ge 1$)}
\Output{Factor pair $(n_1,n_2)$ such that $n = n_1 n_2$ and $n_1 \approx n_2$}

$a \gets \lfloor \sqrt{n} \rfloor$ \tcp*{Start from the square-root scale}
\While{$a > 1$ \textbf{and} $n \bmod a \ne 0$}{
    $a \gets a - 1$ \tcp*{Search the closest divisor below $\sqrt{n}$}
}
$n_2 \gets a$\;
$n_1 \gets n / n_2$\;
\Return $(n_1,n_2)$\;

\end{algorithm}

\begin{algorithm*}[h!]
\caption{Quantization Framework of BWLA}
\label{alg:okt_psp}
\DontPrintSemicolon
\SetKwInOut{Input}{Input}
\SetKwInOut{Output}{Output}

\Input{
Pretrained weight matrix $W \in \mathbb{R}^{n\times m}$ with $n = n_1 n_2$; \\
rank $k$ for low-rank correction; proximal parameter $\mu$; \\
number of outer iterations $T_{\text{outer}}$; Kronecker factors $(n_1,n_2)$
}
\Output{
Orthogonal factors $R_1 \in \mathbb{R}^{n_1\times n_1}$, $R_2 \in \mathbb{R}^{n_2\times n_2}$; \\
low-rank correction $M=AB$; bimodal parameters $\Theta=\{c_i,\sigma_i^2\}$; \\
binarized weights $\tilde W$ and per-channel scales $(\alpha,\beta)$
}

\textbf{Initialization:}\;
Initialize $R_1,R_2$ as identity or random orthogonal matrices\;
Set $M_0\gets\mathbf 0$; initialize $c_i,\sigma_i^2$ from $W$ (e.g., empirical stats)\;
Define $H \gets I - \tfrac1m \mathbf 1 \mathbf 1^\top$\;

\For{$t \gets 0$ \KwTo $T_{\text{outer}}-1$}{
  \tcc{\textcolor{blue}{OKT: orthogonal--Kronecker rotation and row-wise centering}}
  $R \gets R_1 \otimes R_2$;\quad
  $X \gets (W - M_t)\,R\,H$\;

  \tcp{EM update of symmetric two-component GMM (row-wise)}
  \For{each row $i = 1,\dots,n$}{
    Compute responsibilities $r_{ij}^+,r_{ij}^-$ for all $j$ using $x_{ij}$, $c_i$, $\sigma_i^2$\;
    Update
    $c_i \gets \frac1m \sum_j (2r_{ij}^+-1)x_{ij}$,
    $\sigma_i^2 \gets \frac1m \sum_j\bigl[r_{ij}^+(x_{ij}-c_i)^2 + r_{ij}^-(x_{ij}+c_i)^2\bigr]$\;
    Enforce $\sigma_i \gets \max(\sigma_i,\sigma_{\min})$ and set $\lambda_i \gets \sigma_i^{-2}$\;
  }
  Optionally compute $\bar r_i = \tfrac1m\sum_j r_{ij}^+$ and regularizer 
  $(\tfrac1n\sum_i \bar r_i - \tfrac12)^2$\;

  \tcc{\textcolor{red}{PSP: proximal SVD-based low-rank correction}}
  Compute gradient $G_X = \partial L / \partial X$ and adjoint
  $G \gets G_X H^\top R^\top$\;
  Form proximal point $Y_t \gets M_t - \tfrac1\mu G$\;
  Compute rank-$k$ truncated SVD $Y_t \approx U_k \Sigma_k V_k^\top$\;
  Set $M_{t+1} \gets U_k \Sigma_k V_k^\top$ (optionally $A_{t+1} \gets U_k\Sigma_k^{1/2}$,
  $B_{t+1} \gets \Sigma_k^{1/2}V_k^\top$)\;

  \tcp{MM: alternating Procrustes updates for $R_1,R_2$}
  Recompute $R \gets R_1\otimes R_2$ and $X \gets (W - M_{t+1})RH$ if needed\;
  Build reshaped pairs $(V_{\mathrm{mat},i},M_i)$ from rows of $(W-M_{t+1})R$ and targets
  $m_{ij} = (2r_{ij}^+-1)c_i$\;
  With fixed $R_2$, form
  $C_1 \gets \sum_i \lambda_i V_{\mathrm{mat},i}^\top R_2 M_i$,
  compute SVD $C_1 = U_1\Sigma_1V_1^\top$ and update $R_1\gets U_1V_1^\top$\;
  With fixed $R_1$, form
  $C_2 \gets \sum_i \lambda_i V_{\mathrm{mat},i} R_1 M_i^\top$,
  compute SVD $C_2 = U_2\Sigma_2V_2^\top$ and update $R_2\gets U_2V_2^\top$\;
}

\textbf{Final binarization in OKT+PSP coordinate:}\;
$R \gets R_1\otimes R_2$, $M \gets M_{T_{\text{outer}}}$, $X \gets (W-M)RH$\;
\For{each channel $j = 1,\dots,m$}{
  $\beta_j \gets \tfrac1n \sum_i X_{ij}$,\quad
  $\alpha_j \gets \tfrac1n \sum_i |X_{ij}-\beta_j|$\;
  $\tilde W_{ij} \gets \operatorname{Sign}(X_{ij}-\beta_j)$,\quad
  $W_{\mathrm{deq},ij} \gets \tilde W_{ij}\alpha_j + \beta_j$\;
}
\Return{$R_1,R_2,A,B,\Theta,\tilde W,(\alpha,\beta)$}\;

\end{algorithm*}

%% file: appendix/A.2.tex
\subsection{Detailed Proofs and Derivations}
\label{sec:proof}

\subsubsection{Proof of Theorem~\ref{theo1}}
\label{app:proof-aux}

For completeness we restate the theorem in a slightly more formal way.

\begin{theorem*}[Restatement of Theorem~\ref{theo1}]
Let $W\in\mathbb{R}^{n\times m}$ be a weight matrix whose $i$-th row
$w_i^\top\in\mathbb{R}^{1\times m}$ is sampled independently from a
(unimodal) Gaussian distribution
\(
w_i \sim \mathcal N(\mu_i, \sigma_i^2 I_m),
\)
where $\mu_i\in\mathbb{R}^m$ and $\sigma_i>0$ for all $i\in\{1,\dots,n\}$.
Then there exists a learnable orthogonal matrix
$R\in\mathcal O(m)=\{Q\in\mathbb{R}^{m\times m}:Q^\top Q = I_m\}$ such that
the transformed weights
\(
W' = W R
\)
admit a bimodal distribution across channels that can be represented (up to
a KL--optimal approximation) as a two--component Gaussian mixture,
\begin{equation}
  x \sim \pi\,\mathcal N(\mu_1,\sigma_1^2)
      + (1-\pi)\,\mathcal N(\mu_2,\sigma_2^2),
  \label{eq:aux-mixture-appendix}
\end{equation}
for some mixing coefficient $\pi\in(0,1)$ and parameters
$\mu_1,\mu_2\in\mathbb{R}$, $\sigma_1,\sigma_2>0$.
\end{theorem*}

\paragraph{Step 1: Effect of an orthogonal transformation on Gaussian rows.}

We first record a standard fact about multivariate Gaussians.

\begin{lemma}
\label{lem:gaussian-orthogonal}
Let $X\sim\mathcal N(\mu,\Sigma)$ in $\mathbb{R}^m$ and let
$R\in\mathcal O(m)$ be any orthogonal matrix.
Then the transformed random vector $X' = X R$ satisfies
\(
X' \sim \mathcal N(\mu',\Sigma'),
\)
where $\mu' = \mu R$ and $\Sigma' = R^\top \Sigma R$.
In particular, $X'$ is still Gaussian and therefore unimodal.
\end{lemma}

\begin{proof}[Proof of Lemma~\ref{lem:gaussian-orthogonal}]
By linearity,
\(
\mathbb E[X'] = \mathbb E[X R] = \mu R.
\)
Moreover,
\begin{equation}
\begin{split}
\mathrm{Cov}(X')
=& \mathbb E\big[(X'-\mu')(X'-\mu')^\top\big]\\
=& \mathbb E\big[R^\top (X-\mu) (X-\mu)^\top R\big]
= R^\top \Sigma R.
\end{split}
\end{equation}
The Gaussianity of $X'$ follows from the fact that any linear image of a
Gaussian random vector is again Gaussian.
\end{proof}

Applying Lemma~\ref{lem:gaussian-orthogonal} row--wise to $W$ we obtain:
for any $R\in\mathcal O(m)$ and any $i\in\{1,\dots,n\}$,
\begin{equation}
\begin{split}
  w'_i := w_i R& \sim \mathcal N(\mu'_i,\Sigma'_i),
  \\
  \mu'_i = \mu_i R,\quad \Sigma'_i = &R^\top (\sigma_i^2 I_m) R
                                 = \sigma_i^2 I_m.
  \label{eq:row-transform-gaussian}
\end{split}
\end{equation}
Thus each \emph{row} of $W' = W R$ remains Gaussian, but its mean
orientation in $\mathbb{R}^m$ is modulated by the choice of $R$.

\paragraph{Step 2: Empirical channel distribution induced by $R$.}

The bimodality in Theorem~\ref{theo1} is defined
over the empirical distribution of scalar weights \emph{across channels}.
Formally, for a given orthogonal matrix $R$ let
\begin{equation}
  W' = W R, \qquad W' = (w'_{ij})_{1\le i\le n,\,1\le j\le m}.
\end{equation}
Define the empirical distribution of the transformed scalar entries as
\begin{equation}
  p_R(x)
  := \frac{1}{nm} \sum_{i=1}^n \sum_{j=1}^m
     \delta\big(x - w'_{ij}\big),
  \label{eq:empirical-dist}
\end{equation}
where $\delta(\cdot)$ denotes the Dirac measure.  Intuitively, $p_R$
captures how weights are distributed \emph{across all channels} after the
transformation by $R$. Different choices of $R$ change the projections
$w'_{ij}$ and hence deform $p_R$.

\paragraph{Step 3: Target bimodal distribution and mixture family.}

We model the desired bimodal distribution by the parametric family of
two--component Gaussian mixtures
\begin{equation}
\begin{split}
  g_\theta(x)
  := \pi\,\mathcal N(x;\mu_1,\sigma_1^2)
     + (1-\pi)\,\mathcal N(x;\mu_2,\sigma_2^2),
  \\
  \theta := (\pi,\mu_1,\mu_2,\sigma_1^2,\sigma_2^2),
  \label{eq:gaussian-mixture-family}
\end{split}
\end{equation}
where $\pi\in(0,1)$ and $\mu_1,\mu_2\in\mathbb{R}$,
$\sigma_1,\sigma_2>0$ are free parameters.  This family is rich enough to
represent a broad class of bimodal shapes; in particular, by varying
$\theta$ we can realize any pair of separated Gaussian modes.

\paragraph{Step 4: Joint optimization over $R$ and mixture parameters.}

We now show that there exists a pair $(R^\star,\theta^\star)$ such that the
resulting empirical distribution $p_{R^\star}$ is optimally represented by
a mixture $g_{\theta^\star}$ in the Kullback--Leibler (KL) sense.

Consider the joint optimization problem
\begin{equation}
\begin{split}
  (R^\star,\theta^\star)
  \in \arg\min_{R\in\mathcal O(m),\;\theta\in\Theta}
      L(R,\theta),
  \\
  L(R,\theta) := \mathrm{KL}\big(p_R \,\|\, g_\theta\big),
  \label{eq:joint-kl-min}
\end{split}
\end{equation}
where $\Theta := (0,1)\times\mathbb R^2\times(0,\infty)^2$ is the
parameter space of the mixture in~\eqref{eq:gaussian-mixture-family}.

\emph{Existence of a minimizer.}
The orthogonal group $\mathcal O(m)$ is a compact manifold.  For the
mixture parameters we restrict attention without loss of generality to a
compact subset
\(
\Theta_{\mathrm{cpt}}\subset\Theta
\)
with bounded means and variances bounded away from $0$ and $\infty$; in
practice this corresponds to standard regularization on mixture parameters.
On $\mathcal O(m)\times\Theta_{\mathrm{cpt}}$ the function $L(R,\theta)$ is
continuous because $p_R$ depends continuously on $R$ (the entries
$w'_{ij}$ are smooth functions of $R$) and
$g_\theta(x)$ is continuous in $\theta$ for every $x$.
By the extreme value theorem, a continuous function on a compact domain
attains its minimum, so there exists at least one pair
$(R^\star,\theta^\star)$ solving~\eqref{eq:joint-kl-min}.

\emph{Interpretation of the minimizer.}
For the minimizing pair $(R^\star,\theta^\star)$ we have
\begin{equation}
\begin{split}
  L(R^\star,\theta^\star)
  = \mathrm{KL}\big(p_{R^\star} \,\|\, g_{\theta^\star}\big)
  \le
    \mathrm{KL}\big(p_{R} \,\|\, g_{\theta}\big)
    \\
    \text{for all } (R,\theta)\in\mathcal O(m)\times\Theta_{\mathrm{cpt}}.
  \label{eq:kl-optimal}
\end{split}
\end{equation}
That is, among all orthogonal transformations and two--Gaussian mixtures,
$(R^\star,\theta^\star)$ provides the best (KL--optimal) two--component
Gaussian approximation of the empirical distribution $p_{R^\star}$.
Equivalently, the scalar entries of
\(
W' := W R^\star
\)
are distributed according to a density that is optimally represented by the
mixture $g_{\theta^\star}$.

In practice, our training procedure implements a stochastic approximation
to~\eqref{eq:joint-kl-min}: the orthogonal matrix $R$ is parameterized (for
example via a skew--symmetric matrix and a Cayley transform), the mixture
parameters $\theta$ are learned jointly, and gradient descent is used to
decrease $L(R,\theta)$.

\paragraph{Step 5: Conclusion.}

By construction, the minimizer $(R^\star,\theta^\star)$ satisfies that the
empirical channel distribution of $W' = W R^\star$ is bimodal and can be
represented by the two--component Gaussian mixture $g_{\theta^\star}$ in the
sense of~\eqref{eq:gaussian-mixture-family}.  Using the standard abuse of
notation in probabilistic modeling, we may summarize this as
\begin{equation}
  w'_{ij} \sim
  \pi^\star \mathcal N(\mu_1^\star,(\sigma_1^\star)^2)
  + (1-\pi^\star)\mathcal N(\mu_2^\star,(\sigma_2^\star)^2),
\end{equation}
which is exactly the claim in
Theorem~\ref{theo1}. Therefore, there exists a
learnable orthogonal matrix $R^\star$ such that the transformed weight
matrix $W' = W R^\star$ exhibits the desired bimodal (two--Gaussian)
behavior across channels.
\qedhere

\paragraph{Interpretation (Modeling Corollary).}
Theorem~\ref{theo1} guarantees that there exists an
orthogonal transform $R^\star$ such that the \emph{overall} empirical
distribution $p_{R^\star}$ of all scalar entries of $W' = W R^\star$ is
well approximated by a two--component Gaussian mixture
$g_{\theta^\star}$. In our modeling, we therefore regard the collection of
scalar weights $\{w'_{ij}\}$ as samples drawn from this learned bimodal
distribution $g_{\theta^\star}$. Consequently, each row vector $w'_i$ can be
viewed as consisting of $m$ (approximately) independent samples from the
same mixture:
\begin{equation}
  w'_{ij} \sim g_{\theta^\star}, \qquad j=1,\dots,m.
\end{equation}
For notational simplicity, we sometimes write
\begin{equation}
  w'_i \sim
  \pi^\star \mathcal N(\mu_1^\star,(\sigma_1^\star)^2)
  + (1-\pi^\star)\mathcal N(\mu_2^\star,(\sigma_2^\star)^2),
\end{equation}
with the understanding that this notation refers to the fact that \emph{the
entries of $w'_i$ are modeled as being drawn from the same learned
bimodal distribution}. This modeling view is sufficient for our subsequent
quantization analysis.

\subsubsection{Details of the Orthogonal--Kronecker Transformation}
\label{app:okt-em-mm}

In this appendix, we present complete derivations of the optimization procedure used to learn the Orthogonal--Kronecker Transformation (OKT). We begin by detailing the symmetric two-component Gaussian mixture model employed to shape each centered row after transformation, followed by exact EM updates for the mixture parameters. We then establish a majorization--minimization (MM) scheme that yields closed-form orthogonal Procrustes updates for the Kronecker factors $R_1$ and $R_2$. We conclude with complexity considerations and modeling remarks.

\paragraph{Symmetric Two-Component GMM for Centered Rows.}
Given an orthogonal matrix $R$, each transformed row is centered as
\begin{equation}
x_i = (w_i R) - \tfrac{1}{m}\mathbf{1}\mathbf{1}^\top (w_i R),
\qquad i=1,\dots,n,
\end{equation}
which ensures that the scalar entries $\{x_{ij}\}_{j=1}^m$ are approximately symmetric about zero. To encourage bimodal structure aligned with binary quantization, each entry is modeled using the symmetric two-component mixture
\begin{equation}
\begin{split}
p(x_{ij}\mid c_i,\sigma_i^2)
=& \\
\tfrac12\,\phi(x_{ij};&+c_i,\sigma_i^2)
+ \tfrac12\,\phi(x_{ij};-c_i,\sigma_i^2),
\end{split}
\label{eq:app-symm-gmm}
\end{equation}
where $\Theta_i=\{c_i,\sigma_i^2\}$ are rowwise parameters. The normalized negative log-likelihood is
\begin{equation}
\label{eq:app-lgmm}
\begin{split}
L_{\mathrm{GMM}}(R,\Theta)
=& -\frac{1}{nm}\sum_{i=1}^{n}\sum_{j=1}^{m}\\
\log\Big[
\tfrac12\,\phi(&x_{ij};+c_i,\sigma_i^2)
+ \tfrac12\,\phi(x_{ij};-c_i,\sigma_i^2)
\Big].
\end{split}
\end{equation}
This serves as the primary objective governing bimodality.

\paragraph{EM Updates for Mixture Parameters.}
For fixed $R$, we optimize $\Theta$ using EM. Introduce latent variables $z_{ij}\in\{+1,-1\}$ denoting mixture component assignments. The posterior responsibility of the positive component is
\begin{equation}
\label{eq:app-resp}
r_{ij}^+
= \frac{\phi(x_{ij};c_i,\sigma_i^2)}
{\phi(x_{ij};c_i,\sigma_i^2)+\phi(x_{ij};-c_i,\sigma_i^2)}.
\end{equation}

The complete-data log-likelihood yields closed-form M-step updates. Using standard EM derivations,
\begin{equation}
\label{eq:app-ci-update}
c_i^{\text{new}} = \frac{1}{m}\sum_{j=1}^m (2r_{ij}^+-1)x_{ij},
\end{equation}
\begin{equation}
\label{eq:app-sigma-update}
(\sigma_i^2)^{\text{new}}
= \frac{1}{m}\sum_{j=1}^m
\Big[
r_{ij}^+(x_{ij}-c_i)^2 + (1-r_{ij}^+)(x_{ij}+c_i)^2
\Big],
\end{equation}
with $\sigma_i^2$ bounded below by $\sigma_{\min}^2$ to prevent collapse. Classic EM theory guarantees monotonic improvement of $L_{\mathrm{GMM}}$ for fixed $R$.

For later use in MM, the gradient of $L_{\mathrm{GMM}}$ w.r.t.\ scalar entries is
\begin{equation}
\label{eq:app-grad-x}
\frac{\partial L_{\mathrm{GMM}}}{\partial x_{ij}}
= \frac{1}{nm\,\sigma_i^2}
\Big[
r_{ij}^{+}(x_{ij}-c_i)
+ (1-r_{ij}^{+})(x_{ij}+c_i)
\Big],
\end{equation}
indicating an attraction toward the nearest mode.

\paragraph{MM Surrogate for Orthogonal Updates.}
Directly optimizing $R$ in $L_{\mathrm{GMM}}$ is difficult due to the log-sum-exp term. We construct a quadratic majorization surrogate by leveraging the convex property of the negative log-likelihood function $f(y) = -\log(y)$. Let $y_{ij}$ be the argument of the logarithm in Eq.~\eqref{eq:app-lgmm}. The key identity is $\log(a+b) \le \log(a) + b/a - 1$. Applying this and substituting the first-order Taylor expansion yields a quadratic majorizer $\widetilde{L}(R)$.
Specifically, we use the property that the negative log-likelihood is majorized by the weighted least squares error between the transformed data $x_{ij}$ and the posterior mean of the latent component $m_{ij}$:
\begin{equation}
\label{eq:app-mij}
m_{ij} = (2r_{ij}^+ - 1)c_i,
\end{equation}
where $m_{ij}$ is the posterior mean of the latent component. Define the MM surrogate
\begin{equation}
\label{eq:app-mm-surrogate}
\widetilde{L}(R)
= \frac{1}{nm}\sum_{i=1}^n \lambda_i
\sum_{j=1}^m (x_{ij} - m_{ij})^2,
\end{equation}
where $\lambda_i=\sigma_i^{-2}$. This surrogate satisfies the MM properties:
\begin{equation}
\begin{split}
L_{\mathrm{GMM}}(R,\Theta)
\le \widetilde{L}(R) + \text{const},
\\
\widetilde{L}(R^{(t)}) = L_{\mathrm{GMM}}(R^{(t)},\Theta),
\end{split}
\end{equation}
which is guaranteed when $\Theta$ is fixed and $\lambda_i$ are derived from the M-step estimates.

\paragraph{Matrix Formulation and Kronecker Structure.}
Reshape each $x_i$ into a matrix $V_{\mathrm{mat},i}\in\mathbb{R}^{n_1\times n_2}$ and $m_i$ into $M_i$ of the same size. Using the Kronecker identity,
\begin{equation}
x_i(R) = \operatorname{vec}\!\bigl(
R_2^\top V_{\mathrm{mat},i} R_1
\bigr)^\top,
\end{equation}
the surrogate becomes
\begin{equation}
\widetilde{L}(R_1,R_2)
= \frac{1}{nm}\sum_{i=1}^n \lambda_i
\big\|
R_2^\top V_{\mathrm{mat},i} R_1 - M_i
\big\|_F^2.
\label{eq:app-mm-matrix}
\end{equation}

\paragraph{Alternating Procrustes Updates.}
Fixing $R_2$, minimizing Eq.~\eqref{eq:app-mm-matrix} over $R_1$ is equivalent to solving an orthogonal Procrustes problem:
\begin{equation}
C_1 = \sum_{i=1}^n \lambda_i V_{\mathrm{mat},i}^{\top} R_2 M_i,
\end{equation}
\begin{equation}
\label{eq:app-R1-update}
U_1\Sigma_1V_1^\top = \mathrm{SVD}(C_1),\qquad
R_1^{\text{new}} = U_1 V_1^\top.
\end{equation}

Similarly, fixing $R_1$ gives
\begin{equation}
C_2 = \sum_{i=1}^n \lambda_i V_{\mathrm{mat},i} R_1 M_i^\top,
\end{equation}
\begin{equation}
\label{eq:app-R2-update}
U_2\Sigma_2V_2^\top = \mathrm{SVD}(C_2),\qquad
R_2^{\text{new}} = U_2 V_2^\top.
\end{equation}

\paragraph{Monotonicity.}
Because each Procrustes update globally minimizes its corresponding MM subproblem, we obtain
\begin{equation}
\widetilde{L}^{(t+1)} \le \widetilde{L}^{(t)},
\qquad
L_{\mathrm{GMM}}^{(t+1)} \le L_{\mathrm{GMM}}^{(t)},
\end{equation}
ensuring monotonic descent of the true objective.

\paragraph{Complexity Analysis and Remarks.}
A full dense $m\times m$ orthogonal update costs $O(m^2)$ parameters and FLOPs. Under the OKT parameterization with $m=n_1 n_2$, computing $R_2^\top V_{\mathrm{mat},i}R_1$ requires
\begin{equation}
O(n_1^2 n_2 + n_1 n_2^2) = O\!\left(n_1 n_2 (n_1+n_2)\right).
\end{equation}
When $n_1\approx n_2\approx\sqrt{m}$, this becomes
\begin{equation}
O(m^{3/2}),
\end{equation}
strictly better than $O(m^2)$ for large $m$. The SVDs in Eqs.~\eqref{eq:app-R1-update} and~\eqref{eq:app-R2-update} cost $O(n_1^3+n_2^3)=O(m^{3/2})$ in the balanced case as well.

Regarding modeling assumptions: the rowwise independence and symmetric two-component form are not intended as exact generative models of real LLM weights, but rather as tractable approximations that (i) encourage the emergence of two symmetric modes, (ii) admit closed-form EM updates, and (iii) yield an MM surrogate that decomposes neatly under the Kronecker structure. Empirically, this approximation is sufficient to induce bimodalization and improve quantization robustness at scale.

\subsubsection{Detailed Derivations of the Proximal SVD Projection (PSP)}
\label{app:psp-details}

This appendix provides a complete mathematical derivation of the Proximal SVD Projection (PSP), which acts as a refinement step following the Orthogonal--Kronecker Transformation (OKT). While OKT aligns the weight matrix with a symmetric bimodal coordinate system, structured components may remain misaligned with the target centers $\pm c_i$. PSP removes such components through a rank-constrained proximal update. We derive the adjoint operator, the first-order expansion, the proximal majorizer, the rank-$k$ projection solution, and monotonicity guarantees. Additional modeling considerations are also provided to address potential concerns from reviewers.

\paragraph{Residual Modeling and OKT-Centered Representation.}
Let $M=AB$ be a learnable rank-$k$ residual with
\begin{equation}
A\in\mathbb{R}^{\mathrm{oc}\times k}, \qquad
B\in\mathbb{R}^{k\times\mathrm{ic}}, \qquad
k\ll\min\{\mathrm{oc},\mathrm{ic}\}.
\end{equation}
The corrected weight is $W_{\mathrm{res}} = W - M$. The OKT-centered representation is defined by
\begin{equation}
T_R(U) = (U R)\,H,
\qquad 
H = I - \tfrac{1}{m}\mathbf{1}\mathbf{1}^\top,
\end{equation}
which performs orthogonal mixing followed by row-wise centering. The GMM likelihood is evaluated on
\begin{equation}
X = T_R(W - M).
\end{equation}
Because $M$ contributes only a deterministic affine shift and centering preserves symmetry, the EM updates for the GMM parameters remain unchanged:
\begin{equation}
L_{\mathrm{GMM}}(\Theta^{(t+1)};R,A,B)
\le
L_{\mathrm{GMM}}(\Theta^{(t)};R,A,B).
\end{equation}

\paragraph{Adjoint of the OKT Transform.}
Let $G_X = \partial L / \partial X$ be the gradient in the OKT-centered space. Perturbing the residual gives
\begin{equation}
\delta X = -T_R(\delta M) = -(\delta M\,R)H.
\end{equation}
The variation of the loss is
\begin{equation}
\delta L
= \langle G_X,\delta X \rangle_F
= -\langle G_X, (\delta M R)H\rangle_F.
\end{equation}
Using the Frobenius inner-product identity $\langle A,BC\rangle=\langle B^\top A, C\rangle$, we obtain
\begin{equation}
\delta L
= -\langle G_X H^\top R^\top, \delta M\rangle_F.
\end{equation}
Thus the adjoint operator of $T_R$ is
\begin{equation}
G = T_R^{*}(G_X) = G_X H^\top R^\top,
\end{equation}
and the first-order approximation of $L(W-M)$ around $M_t$ is
\begin{equation}
L(W-M) \approx L(W-M_t) + \langle G, M - M_t \rangle.
\end{equation}

\paragraph{Need for a Proximal Majorizer.}
The GMM likelihood contains log-sum-exp terms and is non-convex. Gradient descent is unstable near steep regions when variances approach their lower bound. We therefore construct a proximal majorizer,
\begin{equation}
\begin{split}
Q_{\mu}(M \mid M_t)
=&
L(W-M_t) \\
+ &\langle G,\,M-M_t\rangle
+ \frac{\mu}{2}\|M-M_t\|_F^2,
\end{split}
\end{equation}
where $\mu>0$ is chosen such that
\begin{equation}
\begin{split}
Q_{\mu}(M_t|M_t)&=L(W-M_t),
\\
Q_{\mu}(M|M_t)&\ge L(W-M) \ \text{locally}.
\end{split}
\end{equation}
The existence of such $\mu$ follows from Lipschitz continuity of $\nabla_M L$, which is guaranteed by enforcing $\sigma_i^2 \ge \sigma_{\min}^2$.

\paragraph{Rank-Constrained Proximal Update.}
The next iterate is obtained by minimizing $Q_{\mu}$ under the rank constraint:
\begin{equation}
M_{t+1}
= \arg\min_{\operatorname{rank}(M)\le k}
Q_{\mu}(M \mid M_t).
\end{equation}
Dropping constants and defining the proximal gradient point
\begin{equation}
Y_t = M_t - \tfrac{1}{\mu}G,
\end{equation}
the subproblem becomes
\begin{equation}
M_{t+1}
=
\arg\min_{\operatorname{rank}(M)\le k}
\|M - Y_t\|_F^2.
\end{equation}

\paragraph{Closed-Form Solution: Truncated SVD.}
By the Eckart--Young--Mirsky theorem, the optimal solution is given by the rank-$k$ truncated SVD of $Y_t$:
\begin{equation}
Y_t = U_k \Sigma_k V_k^\top,
\qquad
M_{t+1} = U_k \Sigma_k V_k^\top.
\end{equation}
Any valid factorization may be used; a stable choice is
\begin{equation}
A_{t+1}=U_k \Sigma_k^{1/2},
\qquad
B_{t+1}=\Sigma_k^{1/2} V_k^\top.
\end{equation}

\paragraph{Monotonicity Guarantee.}
Since $M_{t+1}$ minimizes $Q_\mu$,
\begin{equation}
Q_{\mu}(M_{t+1}|M_t)
\le
Q_{\mu}(M_t|M_t)
= L(W-M_t).
\end{equation}
Because $Q_\mu$ majorizes $L$,
\begin{equation}
L(W-M_{t+1})
\le
Q_{\mu}(M_{t+1}|M_t)
\le
L(W-M_t),
\end{equation}
ensuring monotonic descent.

\paragraph{Computational Complexity.}
The cost is dominated by computing the rank-$k$ truncated SVD of $Y_t$, which can be done efficiently with randomized SVD in
\begin{equation}
O(k\,\mathrm{nnz}(Y_t))\ \text{time}.
\end{equation}
Since $k$ is small (e.g., rank raito = 0.005), PSP adds negligible overhead relative to LLM forward passes.

\paragraph{Modeling Assumptions and Robustness.}
\begin{itemize}
\item The low-rank structure matches the empirical observation that post-OKT residuals lie in low-dimensional subspaces.
\item Symmetry is preserved because centering commutes with $M$.
\item The proximal term stabilizes optimization near steep likelihood regions.
\item PSP does not affect functional equivalence; it modifies only the weight geometry prior to quantization.
\end{itemize}
Altogether, PSP provides a principled, low-cost refinement that eliminates structured residual energy not removable by orthogonal transformations, ensuring a more pronounced bimodal distribution favorable for binarization.

\subsubsection{Why Per-Channel Double-Bell Distributions Are Better for Binarization}
\label{app:double-bell-binary}

In this subsection we connect the OKT--PSP optimization described in
Sections~\ref{app:okt-em-mm} and~\ref{app:psp-details} with the binarization
error of each channel. The key observation is that, under standard
per-channel binarization, the optimal approximation error of a row depends
only on the variance of its magnitudes. We then show that the symmetric GMM
objective in Eq.~\eqref{eq:app-lgmm}, together with its MM surrogate
Eq.~\eqref{eq:app-mm-surrogate} and the PSP refinement, monotonically reduces
this variance by driving entries toward a symmetric double-bell pattern
centered at $\pm c_i$.

\paragraph{Per-channel binarization error as magnitude variance.}

Consider a single row $w\in\mathbb{R}^m$ and its per-channel binary
approximation of the form
\begin{equation}
    \hat{w}_j = \alpha\,\mathrm{sign}(w_j),
    \qquad j=1,\dots,m,
\end{equation}
where the scale $\alpha>0$ is shared within the row.
Define $s_j=\mathrm{sign}(w_j)$ and $a_j = |w_j|$ so that $w_j = s_j a_j$.
The per-channel squared error is
\begin{equation}
    \mathcal{E}_w(\alpha)
    = \sum_{j=1}^m (w_j - \hat{w}_j)^2
    = \sum_{j=1}^m (a_j - \alpha)^2.
\end{equation}
Hence the unique minimizer is the sample mean of magnitudes
\begin{equation}
    \alpha^\star
    = \frac{1}{m}\sum_{j=1}^m a_j
    \;\triangleq\; \bar{a},
\end{equation}
and the corresponding minimal error is
\begin{equation}
    \mathcal{E}_w^\star
    = \sum_{j=1}^m (a_j - \bar{a})^2
    = m\,\mathrm{Var}(a_j).
\end{equation}
Thus, for a fixed sign pattern, the per-channel binary approximation error is
completely determined by the variance of magnitudes within the row: the more
concentrated the magnitudes, the smaller the binarization error.

As a special case, if a row takes values only in $\{\pm c\}$ with $c>0$, then
all magnitudes equal $c$, $\mathrm{Var}(a_j)=0$, and the optimal binary
approximation exactly recovers $w$. This is precisely the ideal
double-bell pattern we aim for.

\paragraph{Effect of OKT: GMM-driven tightening around $\pm c_i$.}

Section~\ref{app:okt-em-mm} defines, for a fixed orthogonal $R$, the
centered rows
\begin{equation}
    x_i = (w_i R) - \tfrac{1}{m}\mathbf{1}\mathbf{1}^\top(w_i R),
    \qquad i=1,\dots,n,
\end{equation}
and models the entries $\{x_{ij}\}_{j=1}^m$ using the symmetric GMM
in Eq.~\eqref{eq:app-symm-gmm} with rowwise parameters
$\Theta_i=\{c_i,\sigma_i^2\}$. The normalized negative log-likelihood
$L_{\mathrm{GMM}}(R,\Theta)$ is given in Eq.~\eqref{eq:app-lgmm}. For fixed
$R$, the EM updates in Eqs.~\eqref{eq:app-ci-update}–\eqref{eq:app-sigma-update}
monotonically decrease $L_{\mathrm{GMM}}$ with respect to $\Theta$.

The gradient of $L_{\mathrm{GMM}}$ w.r.t.\ the scalar entries $x_{ij}$,
Eq.~\eqref{eq:app-grad-x}, can be rewritten as
\begin{equation}
    \frac{\partial L_{\mathrm{GMM}}}{\partial x_{ij}}
    = \frac{1}{nm\,\sigma_i^2}
      \Big[
        r_{ij}^{+}(x_{ij}-c_i)
        + (1-r_{ij}^{+})(x_{ij}+c_i)
      \Big].
\end{equation}
This expression shows that $x_{ij}$ is attracted toward the closer center
$\pm c_i$ under the soft assignments $r_{ij}^+,1-r_{ij}^+$: when
$x_{ij}>0$, the gradient predominantly involves $(x_{ij}-c_i)$; when
$x_{ij}<0$, it predominantly involves $(x_{ij}+c_i)$.

To enable tractable updates for the Kronecker factors $R_1,R_2$, we
upper bound the GMM loss by the MM surrogate in
Eq.~\eqref{eq:app-mm-surrogate}, which, after reshaping, becomes
Eq.~\eqref{eq:app-mm-matrix} in matrix form:
\begin{equation}
    \widetilde{L}(R_1,R_2)
    = \frac{1}{nm}\sum_{i=1}^n \lambda_i
    \big\|
        R_2^\top V_{\mathrm{mat},i} R_1 - M_i
    \big\|_F^2,
\end{equation}
where $\lambda_i = \sigma_i^{-2}$ and the target matrices $M_i$ encode the
posterior means $m_{ij}=(2r_{ij}^+ - 1)c_i$ in Eq.~\eqref{eq:app-mij}. Each
Procrustes step
\eqref{eq:app-R1-update}–\eqref{eq:app-R2-update} chooses $(R_1,R_2)$ to make
$R_2^\top V_{\mathrm{mat},i} R_1$ as close as possible to $M_i$ in Frobenius
norm.

Collecting these pieces, the alternating EM--MM scheme over $(\Theta,R)$
jointly:
\begin{itemize}
    \item pulls the transformed entries $x_{ij}$ toward the row-specific
          centers $\pm c_i$ (via Eq.~\eqref{eq:app-grad-x}), and
    \item rotates the coordinate system (via the Kronecker $R_1\otimes R_2$)
          so that the centered rows become better aligned with the
          double-bell targets $M_i$ (via Eq.~\eqref{eq:app-mm-matrix}).
\end{itemize}
As a result, the magnitudes $|x_{ij}|$ concentrate around $c_i$, reducing
their empirical variance and hence lowering the per-channel optimal
binarization error in the rotated coordinates. Orthogonal invariance of the
L$_2$ error then implies the same reduction for the original rows $w_i$.

\paragraph{Effect of PSP: low-rank refinement of double-bell structure.}

Section~\ref{app:psp-details} introduces a low-rank residual
$M = AB$ and defines the corrected weights $W_{\mathrm{res}} = W - M$, with
the OKT-centered representation
\begin{equation}
    X = T_R(W-M),
\end{equation}
where $T_R(\cdot)$ denotes the orthogonal mixing and row-centering. The
GMM loss $L_{\mathrm{GMM}}$ is now evaluated on $X$, and the gradient with
respect to $X$ is propagated back through the adjoint operator
$T_R^*$ to obtain $G = \partial L/\partial M$.

To ensure stable and monotone updates on $M$ under the rank constraint
$\operatorname{rank}(M)\le k$, a proximal majorizer $Q_\mu(M\mid M_t)$
is constructed around the current iterate $M_t$. The next iterate
$M_{t+1}$ is obtained as the best rank-$k$ approximation of the proximal
gradient point
\begin{equation}
    Y_t = M_t - \tfrac{1}{\mu}G,
\end{equation}
via truncated SVD:
\begin{equation}
    Y_t = U_k \Sigma_k V_k^\top,
    \qquad
    M_{t+1} = U_k \Sigma_k V_k^\top.
\end{equation}
By construction,
\begin{equation}
\begin{split}
    L(W-M_{t+1})
    &\le Q_\mu(M_{t+1}\mid M_t)\\
    &\le Q_\mu(M_t\mid M_t)
    = L(W-M_t),
\end{split}
\end{equation}
so each PSP step monotonically decreases the same loss that measures
deviation from the ideal double-bell structure.

Intuitively, PSP absorbs structured residual components that still push
entries away from the learned centers $\pm c_i$ after OKT, but does so in a
rank-constrained fashion. This provides an additional tightening of the
magnitude distribution $|x_{ij}|$ around $c_i$, further reducing the
per-channel binarization error while leaving the layer's functional mapping
intact.

\paragraph{Summary of the double-bell advantage under OKT+PSP.}

Putting everything together:
\begin{itemize}
    \item For each row, the optimal per-channel binary approximation error
          equals $m$ times the variance of its magnitudes.
    \item The OKT EM--MM procedure (Section~\ref{app:okt-em-mm}) explicitly
          drives each centered row toward a symmetric double-bell pattern at
          $\pm c_i$ by minimizing the GMM loss $L_{\mathrm{GMM}}$ and its
          MM surrogate, thereby reducing the variance of magnitudes and the
          binarization error.
    \item The PSP refinement (Section~\ref{app:psp-details}) monotonically
          decreases the same loss under a low-rank constraint, eliminating
          structured residual energy that is misaligned with the double-bell
          geometry and further lowering the variance.
\end{itemize}
Therefore, within our orthogonally equivalent representations of a given
Gaussian-like row, the OKT+PSP optimization explicitly seeks those
representations whose entries cluster tightly around $\pm c_i$, i.e., those
with per-channel double-bell distributions that are provably more amenable to
binarization.

%% file: appendix/A.3.tex
\subsection{More Experimental Results}

\subsubsection{More Detailed Results}
\label{app:detailed main results}
In this section, we provide the full expanded results corresponding to Table~\ref{tab:main_results} in the main paper.
When activations remain in full precision (16-bit), the detailed breakdown in Table~\ref{tab:app:tab1} shows that classical weight-only quantization methods such as GPTQ and OSTQuant consistently experience substantial accuracy degradation across all model scales. More advanced approaches like BiLLM and ARB-LLM are relatively more stable, yet still show noticeable drops—particularly on reasoning-heavy benchmarks such as ARC-C and HellaSwag. In contrast, our BWLA method achieves strong and uniform improvements across all models, recovering 10–25 points compared with ARB-LLM and 20–40 points compared with GPTQ/OSTQ. These results indicate that even when activations are not quantized, BWLA effectively reshapes the weight distribution into a more quantization-friendly geometry, enabling high-fidelity binarization without sacrificing semantic reasoning.

When activations are further quantized to a more challenging 6-bit regime, the performance disparities become significantly amplified, as shown in Table~\ref{tab:app:tab2}. Under this setting, GPTQ and OSTQuant almost completely collapse, typically falling to the 20–30\% accuracy range, while BiLLM and ARB-LLM display partial robustness but still suffer severe degradation of 20–40 points across multiple tasks and models. DBellQuant also exhibits unstable behavior and fails to provide complete results across models. In sharp contrast, BWLA maintains strong and reliable performance across all architectures, delivering large absolute gains over existing methods and narrowing the FP16 gap by 30+ points on several benchmarks. The resilience of BWLA in this low-precision activation setting highlights its ability to jointly stabilize weight–activation interactions—precisely where prior approaches tend to collapse.

Taken together, these detailed results clearly demonstrate the superior robustness and scalability of BWLA. Unlike existing PTQ methods that quickly degrade as activation precision decreases, BWLA remains consistently strong across both 16-bit and 6-bit activation regimes, providing reliable, high-quality quantization for a broad spectrum of model sizes.

\input{tab/app_tab_1}

\input{tab/app_tab_2}

\subsubsection{Ablation study on Calibration Data}
\label{app:dataset}

\begin{figure*}[t!]
    \centering
    \includegraphics[width=0.9\linewidth]{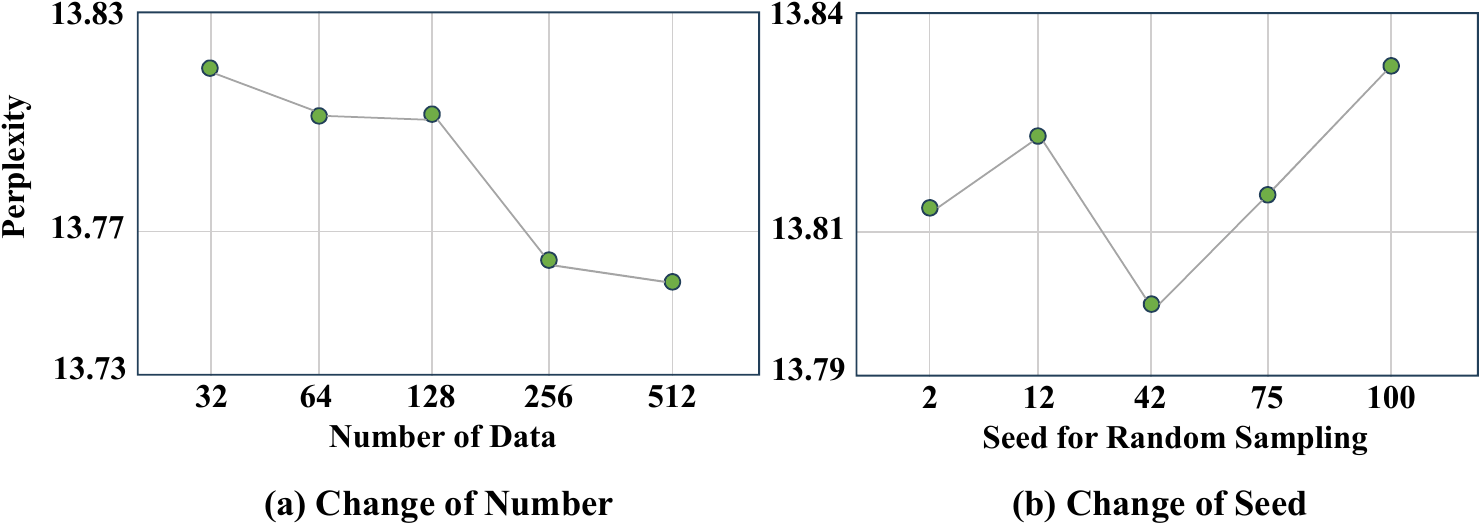}
    \caption{Perplexity of Qwen3-14B using calibration data sampled with different number or seeds from WikiText2.}
    \label{fig:dataset}
\end{figure*}

We further investigated the impact of calibration data on BWLA. Specifically, when binarizing Qwen3-14B on WikiText-2 (under 6-bit activation setting), we fixed all quantization hyperparameters and varied the calibration set by (i) subsampling different numbers of calibration samples and (ii) resampling with multiple random seeds that control both sample selection and token order. For each setting, we re-quantized once and evaluated perplexity on a held-out split. Across all sizes and seeds, the resulting perplexity fluctuates by $<$1\% relative to the mean. As shown in Figure~\ref{fig:dataset}, the curves remain nearly flat as calibration size increases, and seed-wise traces largely overlap, indicating that even our smallest tested calibration subsets perform on par with larger ones. These observations confirm that BWLA is highly robust to calibration data selection.

\subsubsection{Loss Curve of OKT and PSP}
\label{app:loss}

\begin{figure*}
    \centering
    \includegraphics[width=1\linewidth]{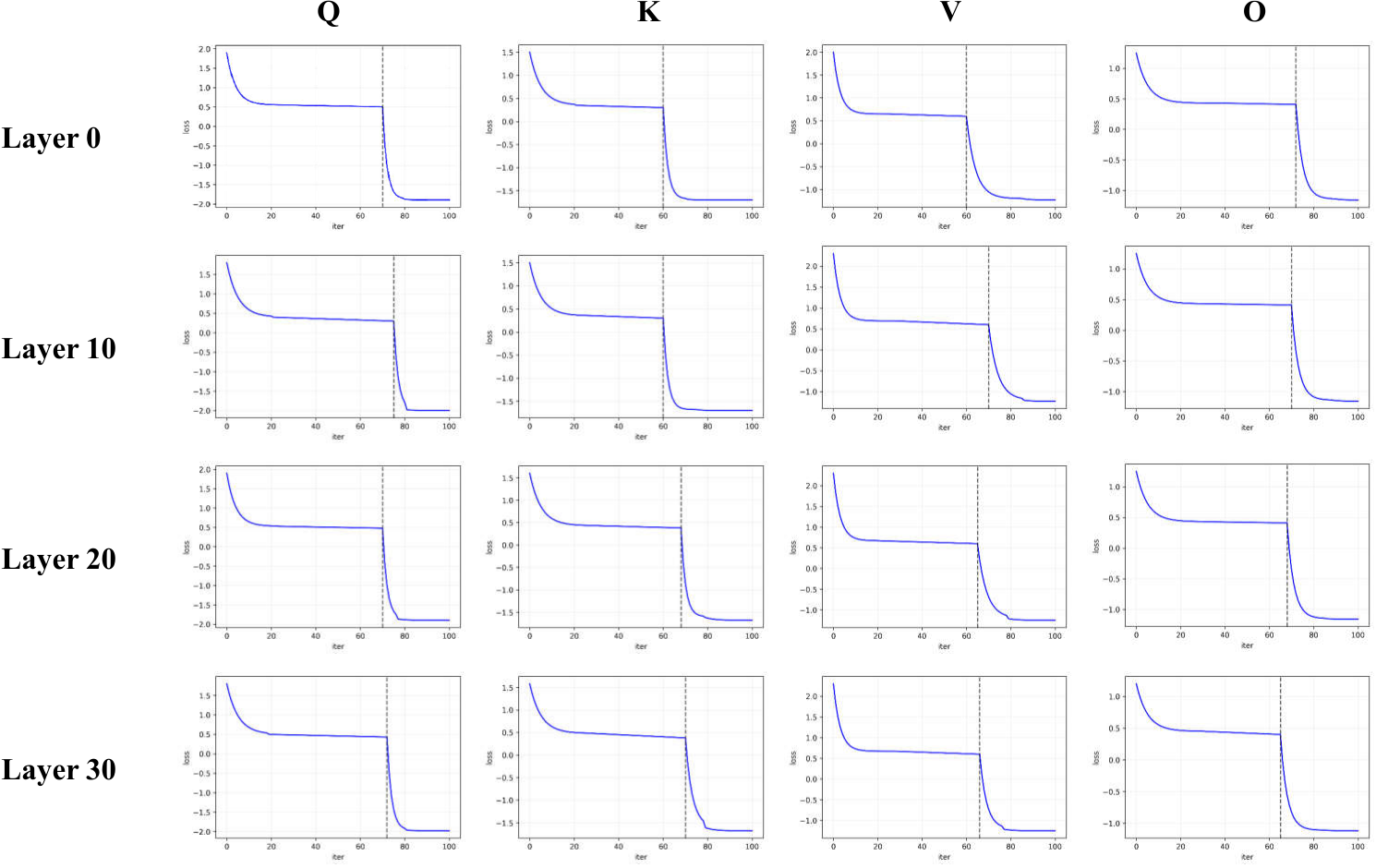}
    \caption{Loss trajectories of the OKT and PSP optimization procedures for the Q, K, V, and O projections in Layers 0, 10, 20, and 30 of LLaMA3-8B. }
    \label{fig:loss_qkvo}
\end{figure*}

\begin{figure*}
    \centering
    \includegraphics[width=0.9\linewidth]{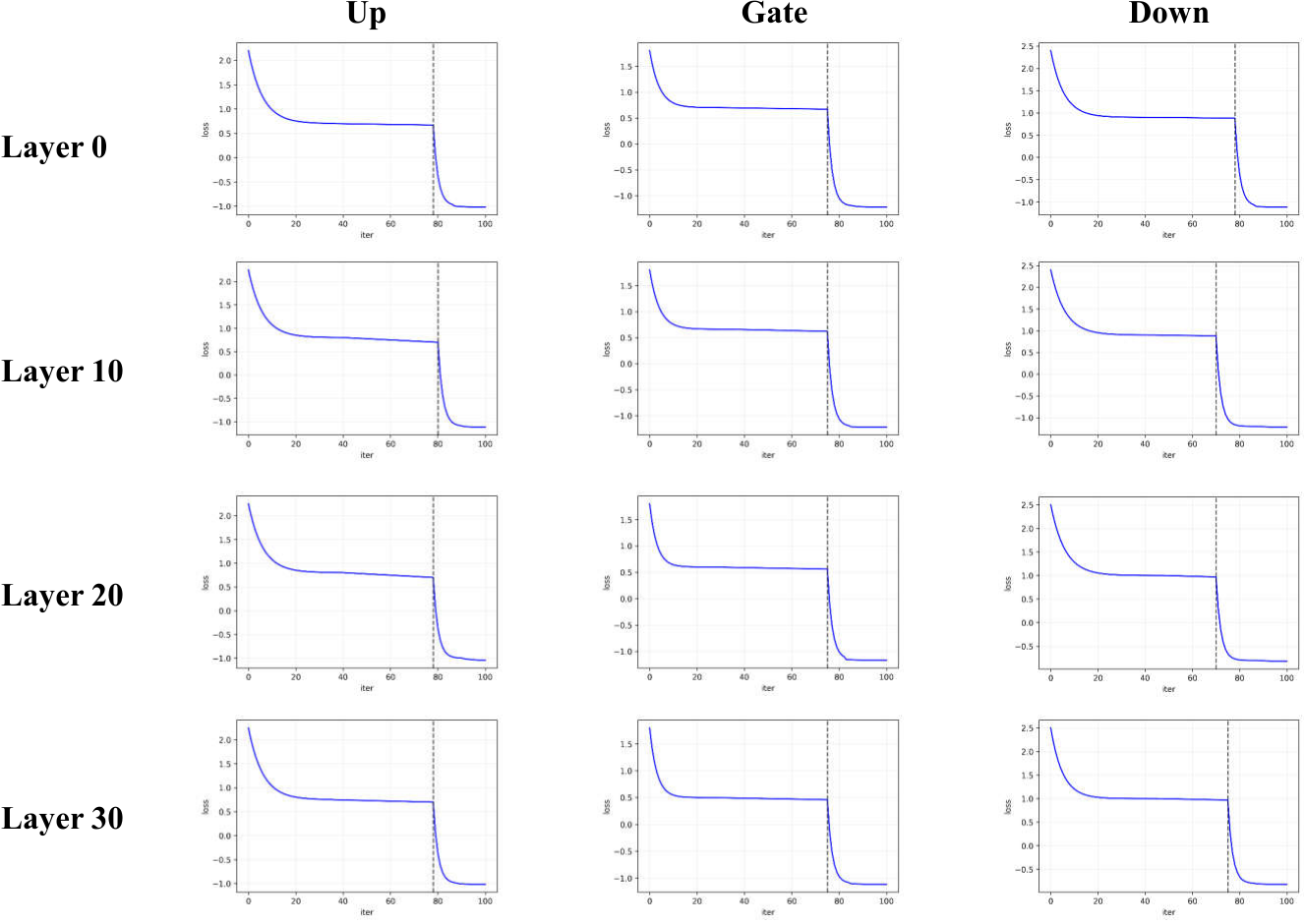}
    \caption{Loss trajectories of the OKT and PSP optimization procedures for the up, gate, and down projections in the MLP blocks of Layers 0, 10, 20, and 30 of LLaMA3-8B.}
    \label{fig:loss_mlp}
\end{figure*}

In this section, we provide the loss trajectories observed during the iterative optimization of the proposed OKT and PSP modules.
As illustrated in Figure~\ref{fig:loss_qkvo} (loss curves of the Q, K, V, and O projections in Layers 0, 10, 20, and 30 of the LLaMA3-8B self-attention blocks) and Figure~\ref{fig:loss_mlp} (loss curves of the up, gate, and down projections in the corresponding MLP blocks), the gray dashed line marks the transition between the two optimization stages: the region to the left corresponds to the OKT phase, while the region to the right corresponds to the PSP phase.
Across all layers and projection types, the loss trajectories exhibit a highly consistent pattern:
the OKT module typically reaches its turning point or enters a stable region within roughly 30 iterations, while PSP converges even faster, usually within 10–15 iterations. Motivated by this observation, we set the total number of optimization steps to 60 throughout our experiments, allocating 40 iterations to OKT and 20 iterations to PSP, ensuring reliable convergence for both stages. These findings demonstrate that, even for the larger transformation matrices within the MLP blocks, both components of our method maintain fast, stable, and well-behaved convergence dynamics across the entire model, providing a robust foundation for the overall quantization procedure.

Table~\ref{tab:quant_time} presents a comparison of the training time required by our method and two representative training-based quantization approaches—OSTQuant, which learns rotation-smoothing matrices, and OmniQuant, which employs block reconstruction—across models of varying sizes. It is important to note that the comparison includes only the optimization or training time, excluding the subsequent GPTQ quantization time. Thanks to the training-free nature of BWLA, which avoids gradient backpropagation entirely, our method completes both the OKT and PSP stages within approximately 60 lightweight iterations. As shown in the table~\ref{tab:quant_time}, BWLA consistently delivers substantial speed advantages over training-based counterparts. For instance, on LLaMA2-7B, BWLA achieves more than a 5× speedup compared to OSTQuant; even on the much larger LLaMA2-70B, the improvement remains above 3×. These results demonstrate that BWLA dramatically reduces optimization cost even at large model scales, offering a far more efficient and scalable alternative to existing training-based quantization methods.

\input{tab/app_tab_3}

%% file: tab/app_tab_1.tex
\begin{table*}[t!]
\centering
\caption{Zero-shot accuracy on Arc-Challenge (AC), Arc-Easy (AE), HellaSwag (HS), LAMBADA-openai (LO), LAMBADA-standard (LS), PIQA (PQ), and WinoGrande (WG) under a \textbf{16-bit activation} quantization setting.}
\label{tab:app:tab1}
\resizebox{\linewidth}{!}{
\tablestyle{3pt}{1.12}
\begin{tabular}{clcccccccccc}
\toprule
\textbf{Model} & \textbf{Method} & \textbf{\#Bits(W)} & \textbf{\#Bits(A)} & \textbf{AE}$\uparrow$ & \textbf{AC}$\uparrow$ & \textbf{HS}$\uparrow$ & \textbf{LO}$\uparrow$ & \textbf{LS}$\uparrow$ & \textbf{PQ}$\uparrow$ & \textbf{WG}$\uparrow$ & \textbf{$\text{Avg.}^7$}$\uparrow$  \\
\midrule
\multirow{7}{*}{LLaMA2-7B} & FP16        & 16   & \multirow{7}{*}{16} & 74.66 & 46.25 & 75.96 & 73.45 & 68.19 & 78.73 & 69.22 & 69.49  \\
\cdashline{2-12}
& GPTQ        & 2    &  & 27.15 & 27.39 & 25.89 & 0.00 & 0.00 & 49.84 & 50.04 & 25.76  \\
& OSTQuant    & 1    &  & 28.15 & 28.12 & 26.33 & 1.31 & 0.88 & 50.81 & 50.12 & 26.53  \\
& BiLLM       & 1.08 &  & 35.48 & 23.55 & 35.18 & 20.30 & 18.03 & 57.83 & 52.49 & 34.69  \\
& ARB-LLM     & 1.08 &  & 40.36 & 24.57 & 36.31 & 22.26 & 18.79 & 58.92 & 55.96 & 36.74  \\
& DBellQuant  & 1.09 &  & 42.85 & 23.89 & 34.82 & --   & --   & 63.98 & 56.27 & --     \\
& \textbf{BWLA} & 1.19 &  & \textbf{54.29} & \textbf{30.80} & \textbf{53.36} & \textbf{53.72} & \textbf{46.21} & \textbf{67.41} & \textbf{61.40} & \textbf{52.46}  \\
\midrule

\multirow{7}{*}{LLaMA2-13B} & FP16        & 16   & \multirow{7}{*}{16} & 77.57 & 49.15 & 79.39 & 76.71 & 70.06 & 80.47 & 71.98 & 72.19  \\
\cdashline{2-12}
& GPTQ        & 2    &  & 24.92 & 28.58 & 25.85 & 0.00 & 0.00 & 49.02 & 50.04 & 25.49  \\
& OSTQuant    & 1    &  & 26.90 & 28.98 & 26.81 & 0.00 & 0.00 & 50.12 & 50.08 & 26.13  \\
& BiLLM       & 1.08 &  & 42.72 & 25.94 & 38.16 & 35.03 & 23.75 & 60.55 & 54.54 & 40.10  \\
& ARB-LLM     & 1.08 &  & 55.30 & 29.01 & 47.33 & 50.88 & 41.24 & 66.92 & 61.64 & 50.33  \\
& DBellQuant  & 1.09 &  & --    & --    & --    & --    & --    & --    & --    & --     \\
& \textbf{BWLA} & 1.19 &  & \textbf{66.84} & \textbf{36.09} & \textbf{63.48} & \textbf{71.32} & \textbf{59.03} & \textbf{74.32} & \textbf{65.59} & \textbf{62.38}  \\
\midrule

\multirow{7}{*}{LLaMA2-70B} & FP16        & 16   & \multirow{7}{*}{16} & 81.02 & 57.34 & 83.78 & 79.57 & 74.67 & 82.70 & 77.90 & 76.71  \\
\cdashline{2-12}
& GPTQ        & 2    &  & 25.33 & 28.12 & 25.99 & 0.00 & 0.02 & 49.02 & 49.80 & 25.47  \\
& OSTQuant    & 1    &  & 26.22 & 29.10 & 26.88 & 0.00 & 0.00 & 50.16 & 50.33 & 26.10  \\
& BiLLM       & 1.08 &  & 63.22 & 38.91 & 57.71 & 46.52 & 36.31 & 67.74 & 66.22 & 53.80  \\
& ARB-LLM     & 1.08 &  & 68.94 & 41.13 & 63.97 & 68.19 & 60.80 & 73.01 & 69.77 & 63.69  \\
& DBellQuant  & 1.09 &  & --    & --    & --    & --    & --    & --    & --    & --     \\
& \textbf{BWLA} & 1.19 &  & \textbf{77.82} & \textbf{51.54} & \textbf{75.90} & \textbf{81.95} & \textbf{75.84} & \textbf{78.78} & \textbf{74.66} & \textbf{73.78}  \\
\midrule

\multirow{7}{*}{LLaMA3-8B} & FP16        & 16   & \multirow{7}{*}{16} & 77.90 & 52.82 & 79.07 & 75.63 & 68.58 & 80.63 & 72.93 & 72.51  \\
\cdashline{2-12}
& GPTQ        & 2    &  & 25.51 & 26.11 & 26.18 & 0.00 & 0.00 & 52.50 & 49.33 & 25.66  \\
& OSTQuant    & 1    &  & 26.50 & 27.18 & 26.53 & 0.00 & 0.00 & 49.78 & 50.08 & 25.72  \\
& BiLLM       & 1.06 &  & 31.10 & 21.93 & 32.52 & 11.59 & 11.66 & 52.29 & 52.80 & 30.56  \\
& ARB-LLM     & 1.06 &  & 41.92 & 24.83 & 36.60 & 28.66 & 19.99 & 59.14 & 56.51 & 38.24  \\
& DBellQuant  & 1.06 &  & 44.11 & 19.97 & 33.19 & --    & --    & 62.35 & 55.96 & --     \\
& \textbf{BWLA} & 1.16 &  & \textbf{52.27} & \textbf{30.80} & \textbf{51.26} & \textbf{45.10} & \textbf{38.56} & \textbf{63.44} & \textbf{60.22} & \textbf{48.81}  \\
\midrule

\multirow{7}{*}{Qwen3-8B} & FP16        & 16   & \multirow{7}{*}{16} & 80.93 & 56.74 & 74.98 & 64.18 & 61.11 & 77.37 & 68.35 & 69.09  \\
\cdashline{2-12}
& GPTQ        & 2    &  & 26.35 & 26.62 & 25.67 & 0.00 & 0.00 & 49.29 & 48.62 & 25.22  \\
& OSTQuant    & 1    &  & 26.88 & 26.96 & 25.93 & 0.00 & 0.00 & 49.41 & 49.69 & 25.55  \\
& BiLLM       & 1.06 &  & 39.31 & 24.06 & 36.28 & 15.99 & 12.05 & 57.13 & 51.14 & 33.71  \\
& ARB-LLM     & 1.06 &  & 49.45 & 28.07 & 42.74 & 33.59 & 25.64 & 62.62 & 55.49 & 42.51  \\
& DBellQuant  & 1.08 &  & --    & --    & --    & --    & --    & --    & --    & --     \\
& \textbf{BWLA} & 1.18 &  & \textbf{56.61} & \textbf{36.09} & \textbf{53.19} & \textbf{48.17} & \textbf{42.65} & \textbf{67.74} & \textbf{60.30} & \textbf{52.11}  \\
\midrule

\multirow{7}{*}{Qwen3-14B} & FP16        & 16   & \multirow{7}{*}{16} & 83.08 & 60.49 & 78.82 & 67.84 & 64.47 & 79.76 & 72.85 & 72.47  \\
\cdashline{2-12}
& GPTQ        & 2    &  & 25.04 & 26.45 & 26.38 & 0.00 & 0.00 & 50.54 & 50.20 & 25.52  \\
& OSTQuant    & 1    &  & 25.88 & 26.73 & 26.70 & 0.00 & 0.00 & 50.88 & 50.21 & 25.77  \\
& BiLLM       & 1.06 &  & 57.45 & 33.87 & 51.28 & 41.39 & 38.93 & 65.72 & 63.38 & 50.29  \\
& ARB-LLM     & 1.06 &  & 60.19 & 36.60 & 51.95 & 48.88 & 44.09 & 66.97 & 63.22 & 53.13  \\
& DBellQuant  & 1.08 &  & --    & --    & --    & --    & --    & --    & --    & --     \\
& \textbf{BWLA} & 1.18 &  & \textbf{68.98} & \textbf{45.31} & \textbf{62.68} & \textbf{64.00} & \textbf{56.36} & \textbf{73.01} & \textbf{68.19} & \textbf{62.65}  \\
\midrule

\multirow{7}{*}{Qwen3-32B} & FP16        & 16   & \multirow{7}{*}{16} & 83.21 & 61.09 & 82.60 & 67.24 & 58.04 & 81.99 & 72.77 & 72.42  \\
\cdashline{2-12}
& GPTQ        & 2    &  & 24.28 & 25.68 & 26.31 & 0.00 & 0.00 & 52.67 & 49.64 & 25.51  \\
& OSTQuant    & 1    &  & 25.61 & 26.83 & 26.94 & 0.00 & 0.00 & 51.78 & 50.64 & 25.97  \\
& BiLLM       & 1.06 &  & 62.25 & 42.15 & 62.75 & 56.63 & 49.23 & 70.62 & 63.54 & 58.17  \\
& ARB-LLM     & 1.06 &  & 74.75 & 51.11 & 65.05 & 66.72 & 60.37 & 72.20 & 69.46 & 65.67  \\
& DBellQuant  & 1.08 &  & --    & --    & --    & --    & --    & --    & --    & --     \\
& \textbf{BWLA} & 1.18 &  & \textbf{73.69} & \textbf{50.33} & \textbf{71.46} & \textbf{71.14} & \textbf{64.87} & \textbf{76.12} & \textbf{70.40} & \textbf{68.29}  \\
\bottomrule
\end{tabular}}
\end{table*}

%% file: tab/app_tab_2.tex
\begin{table*}[h!]
\centering
\caption{Zero-shot accuracy on Arc-Challenge (AC), Arc-Easy (AE), HellaSwag (HS), LAMBADA-openai (LO), LAMBADA-standard (LS), PIQA (PQ), and WinoGrande (WG) under a \textbf{6-bit activation} quantization setting.}
\label{tab:app:tab2}
\resizebox{\linewidth}{!}{
\tablestyle{3pt}{1.12}
\begin{tabular}{clcccccccccc}
\toprule
\textbf{Model} & \textbf{Method} & \textbf{\#Bits(W)} & \textbf{\#Bits(A)} & \textbf{AE}$\uparrow$ & \textbf{AC}$\uparrow$ & \textbf{HS}$\uparrow$ & \textbf{LO}$\uparrow$ & \textbf{LS}$\uparrow$ & \textbf{PQ}$\uparrow$ & \textbf{WG}$\uparrow$ & \textbf{$\text{Avg.}^7$}$\uparrow$  \\
\midrule
\multirow{7}{*}{LLaMA2-7B} 
& FP16        & 16   & 16 & 74.66 & 46.25 & 75.96 & 73.45 & 68.19 & 78.73 & 69.22 & 69.49  \\
\cdashline{2-12}
& GPTQ        & 2    & \multirow{6}{*}{6} & 25.63 & 28.07 & 25.56 & 0.00 & 0.00 & 51.52 & 50.04 & 25.83  \\
& OSTQuant    & 1    &  & 26.11 & 27.01 & 25.33 & 0.00 & 0.00 & 50.52 & 50.01 & 25.57  \\
& BiLLM       & 1.08 &  & 34.01 & 23.38 & 31.69 & 17.43 & 13.02 & 54.46 & 50.75 & 32.11  \\
& ARB-LLM     & 1.08 &  & 36.36 & 22.70 & 33.98 & 23.21 & 18.71 & 56.31 & 52.88 & 34.88  \\
& DBellQuant  & 1.09 &  & 37.50 & 22.18 & 31.94 & --    & --    & 61.10 & 53.85 & 41.31  \\
& \textbf{BWLA} & 1.19 &  & \textbf{45.79} & \textbf{26.02} & \textbf{43.49} & \textbf{32.19} & \textbf{28.49} & \textbf{62.40} & \textbf{55.96} & \textbf{42.05}  \\
\midrule

\multirow{7}{*}{LLaMA2-13B} 
& FP16        & 16   & 16 & 77.57 & 49.15 & 79.39 & 76.71 & 70.06 & 80.47 & 71.98 & 72.19  \\
\cdashline{2-12}
& GPTQ        & 2    & \multirow{6}{*}{6} & 25.93 & 28.58 & 26.44 & 0.00 & 0.02 & 49.02 & 49.80 & 25.68  \\
& OSTQuant    & 1    &  & 25.88 & 28.53 & 25.44 & 0.00 & 0.00 & 50.12 & 50.08 & 25.72  \\
& BiLLM       & 1.08 &  & 38.80 & 23.63 & 31.77 & 21.37 & 17.41 & 58.16 & 51.85 & 34.71  \\
& ARB-LLM     & 1.08 &  & 45.03 & 25.77 & 38.75 & 36.66 & 32.16 & 62.46 & 55.17 & 42.29  \\
& DBellQuant  & 1.09 &  & --    & --    & --    & --    & --    & --    & --    & --     \\
& \textbf{BWLA} & 1.19 &  & \textbf{64.52} & \textbf{37.88} & \textbf{59.92} & \textbf{67.67} & \textbf{56.24} & \textbf{71.60} & \textbf{62.67} & \textbf{60.07}  \\
\midrule

\multirow{7}{*}{LLaMA2-70B} 
& FP16        & 16   & 16 & 81.02 & 57.34 & 83.78 & 79.57 & 74.67 & 82.70 & 77.90 & 76.71  \\
\cdashline{2-12}
& GPTQ        & 2    & \multirow{6}{*}{6} & 25.11 & 27.43 & 25.01 & 0.00 & 0.00 & 49.71 & 50.01 & 25.32  \\
& OSTQuant    & 1    &  & 25.35 & 28.41 & 25.81 & 0.00 & 0.00 & 50.16 & 50.33 & 25.72  \\
& BiLLM       & 1.08 &  & 30.77 & 24.06 & 32.68 & 12.48 & 7.30  & 52.83 & 49.33 & 29.92  \\
& ARB-LLM     & 1.08 &  & 40.70 & 29.27 & 43.19 & 18.38 & 16.81 & 57.02 & 52.25 & 36.80  \\
& DBellQuant  & 1.09 &  & --    & --    & --    & --    & --    & --    & --    & --     \\
& \textbf{BWLA} & 1.19 &  & \textbf{76.73} & \textbf{51.71} & \textbf{73.82} & \textbf{80.09} & \textbf{73.57} & \textbf{78.02} & \textbf{72.69} & \textbf{72.38}  \\
\midrule

\multirow{7}{*}{LLaMA3-8B} 
& FP16        & 16   & 16 & 77.90 & 52.82 & 79.07 & 75.63 & 68.58 & 80.63 & 72.93 & 72.51  \\
\cdashline{2-12}
& GPTQ        & 2    & \multirow{6}{*}{6} & 24.92 & 25.43 & 26.34 & 0.00 & 0.00 & 52.50 & 49.33 & 25.50  \\
& OSTQuant    & 1    &  & 25.62 & 26.44 & 26.11 & 0.00 & 0.00 & 52.38 & 50.30 & 25.84  \\
& BiLLM       & 1.06 &  & 31.86 & 22.70 & 31.75 & 13.60 & 9.80  & 55.77 & 49.72 & 30.74  \\
& ARB-LLM     & 1.06 &  & 36.91 & 24.06 & 35.52 & 22.84 & 19.23 & 56.96 & 52.64 & 35.45  \\
& DBellQuant  & 1.08 &  & 37.63 & 18.77 & 31.83 & --    & --    & 58.00 & 51.92 & 39.63  \\
& \textbf{BWLA} & 1.16 &  & \textbf{48.15} & \textbf{29.18} & \textbf{47.05} & \textbf{38.25} & \textbf{37.18} & \textbf{62.08} & \textbf{58.64} & \textbf{45.79}  \\
\midrule

\multirow{7}{*}{Qwen3-8B} 
& FP16        & 16   & 16 & 80.93 & 56.74 & 74.98 & 64.18 & 61.11 & 77.37 & 68.35 & 69.09  \\
\cdashline{2-12}
& GPTQ        & 2    & \multirow{6}{*}{6} & 25.08 & 22.70 & 25.04 & 0.00 & 0.00 & 48.69 & 49.01 & 24.36  \\
& OSTQuant    & 1    &  & 25.79 & 22.86 & 25.74 & 0.00 & 0.00 & 49.83 & 50.01 & 24.89  \\
& BiLLM       & 1.06 &  & 29.04 & 22.01 & 27.80 & 3.53  & 2.45  & 50.87 & 49.72 & 26.49  \\
& ARB-LLM     & 1.06 &  & 30.30 & 23.46 & 26.65 & 3.42  & 3.01  & 52.12 & 49.88 & 26.98  \\
& DBellQuant  & 1.08 &  & --    & --    & --    & --    & --    & --    & --    & --     \\
& \textbf{BWLA} & 1.18 &  & \textbf{55.30} & \textbf{34.47} & \textbf{50.81} & \textbf{45.26} & \textbf{40.68} & \textbf{66.49} & \textbf{60.22} & \textbf{50.46}  \\
\midrule

\multirow{7}{*}{Qwen3-14B} 
& FP16        & 16   & 16 & 83.08 & 60.49 & 78.82 & 67.84 & 64.47 & 79.76 & 72.85 & 72.47  \\
\cdashline{2-12}
& GPTQ        & 2    & \multirow{6}{*}{6} & 25.08 & 22.70 & 25.04 & 0.00 & 0.00 & 49.51 & 49.57 & 24.56  \\
& OSTQuant    & 1    &  & 25.75 & 22.99 & 25.62 & 0.00 & 0.00 & 49.93 & 49.58 & 24.84  \\
& BiLLM       & 1.06 &  & 37.37 & 28.75 & 37.83 & 0.39  & 1.05  & 58.05 & 49.25 & 30.38  \\
& ARB-LLM     & 1.06 &  & 34.97 & 27.05 & 32.10 & 0.04  & 0.16  & 56.37 & 52.33 & 29.00  \\
& DBellQuant  & 1.08 &  & --    & --    & --    & --    & --    & --    & --    & --     \\
& \textbf{BWLA} & 1.18 &  & \textbf{64.48} & \textbf{42.41} & \textbf{60.36} & \textbf{60.62} & \textbf{53.89} & \textbf{72.20} & \textbf{66.54} & \textbf{60.07}  \\
\midrule

\multirow{7}{*}{Qwen3-32B} 
& FP16        & 16   & 16 & 83.21 & 61.09 & 82.60 & 67.24 & 58.04 & 81.99 & 72.77 & 72.42  \\
\cdashline{2-12}
& GPTQ        & 2    & \multirow{6}{*}{6} & 26.39 & 24.15 & 24.77 & 0.00 & 0.00 & 49.78 & 46.72 & 24.54  \\
& OSTQuant    & 1    &  & 25.40 & 25.36 & 24.79 & 0.00 & 0.00 & 50.93 & 49.95 & 25.20  \\
& BiLLM       & 1.06 &  & 45.20 & 32.08 & 41.24 & 2.46  & 13.14 & 58.43 & 52.80 & 35.05  \\
& ARB-LLM     & 1.06 &  & 34.01 & 23.38 & 31.69 & 17.43 & 13.02 & 54.46 & 50.75 & 32.11  \\
& DBellQuant  & 1.08 &  & --    & --    & --    & --    & --    & --    & --    & --     \\
& \textbf{BWLA} & 1.18 &  & \textbf{72.88} & \textbf{49.75} & \textbf{71.07} & \textbf{70.85} & \textbf{61.89} & \textbf{75.15} & \textbf{68.59} & \textbf{67.17}  \\
\bottomrule
\end{tabular}}
\end{table*}

%% file: tab/app_tab_3.tex
\begin{table*}[t]
\centering
\caption{Comparison of the optimization time of our method with the full training time required by other quantization approaches across models of different sizes.}
\tablestyle{3pt}{1.3}
\begin{tabular}{ccccc}
\hline
Method      & LLaMA2-7B  & LLaMA3-8B  & LLaMA2-13B  & LLaMA2-70B \\ \hline
Omniquant   & 1.6h & 1.8h & 3.3h  & 9.5h \\
OSTQuant    & 0.3h & 0.4h & 0.8h  & 5.5h \\
\textbf{BWLA} &\textbf{0.10h} &\textbf{0.12h} &\textbf{0.25h} &\textbf{1.4h}\\
Speedup     & 3.0$\times$ & 3.3$\times$ & 3.2$\times$  & 3.9$\times$ \\ \hline
\end{tabular}
\label{tab:quant_time}
\end{table*}

%% file: appendix/A.4.tex
\subsection{Distribution Visualizations}
\label{app:distribution}

In this section, we visualize the distributions of all weight matrices and input activations in Layer 12 of the Qwen3-8B model before and after applying BWLA. As shown in Figure~\ref{fig:weight_visual}, before BWLA the weights follow a Gaussian-like unimodal distribution that is unfavorable for binarization. In particular, the down projection exhibits clear outliers whose absolute values exceed 1. After applying the OKT stage, the weight distributions gradually shift from unimodal to bimodal and the extreme outliers are largely removed. With the subsequent PSP refinement, the weights become fully bimodal and clearly aligned with a binary friendly structure, and no visible outliers remain. For example, in the down matrix the original range between the maximum and minimum values is close to 2, whereas after BWLA the range shrinks to about 0.35. A similar effect is observed on the input activations, where heavy tails are significantly suppressed. These observations provide direct evidence that BWLA effectively reshapes both weights and activations into distributions that are much more amenable to accurate binarization.

Figure~\ref{fig:activation_visual} presents the quantile plots of the activation distributions before and after applying BWLA. Prior to processing, the activations exhibit noticeable extreme values that deviate from the main density region, indicating that their original statistical structure is not well suited for low-bit quantization. After being projected through the orthogonal auxiliary matrix obtained from OKT, these extreme values are substantially suppressed, and the overall distribution becomes more concentrated with a significantly tighter tail. A closer examination also reveals that variations across the token dimension become smoother, with no prominent localized anomalies remaining. This demonstrates that OKT effectively reorganizes the geometric layout of the activations, making them more aligned with the statistical properties required for low-bit activation quantization.
Together with the bimodalization of the weights achieved by BWLA, this improvement in activation distribution jointly contributes to the stability and quantizability of the model under binary weights and ultra–low-bit activations.

\begin{figure*}
    \centering
    \includegraphics[width=\linewidth]{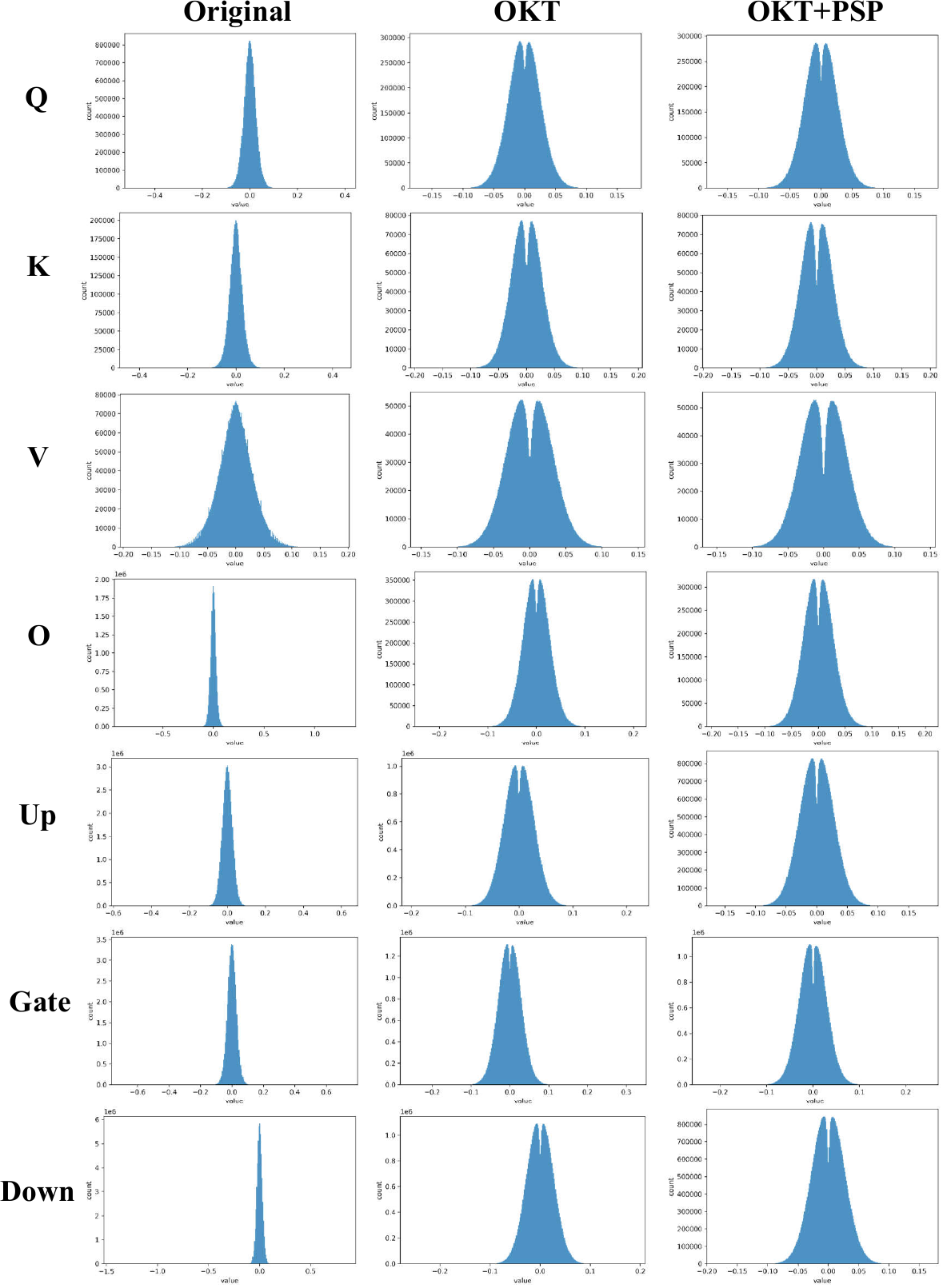}
    \caption{The weight distribution of the 12th layer in Qwen3-8B before and after BWLA.}
    \label{fig:weight_visual}
\end{figure*}

\begin{figure*}
    \centering
    \includegraphics[width=\linewidth]{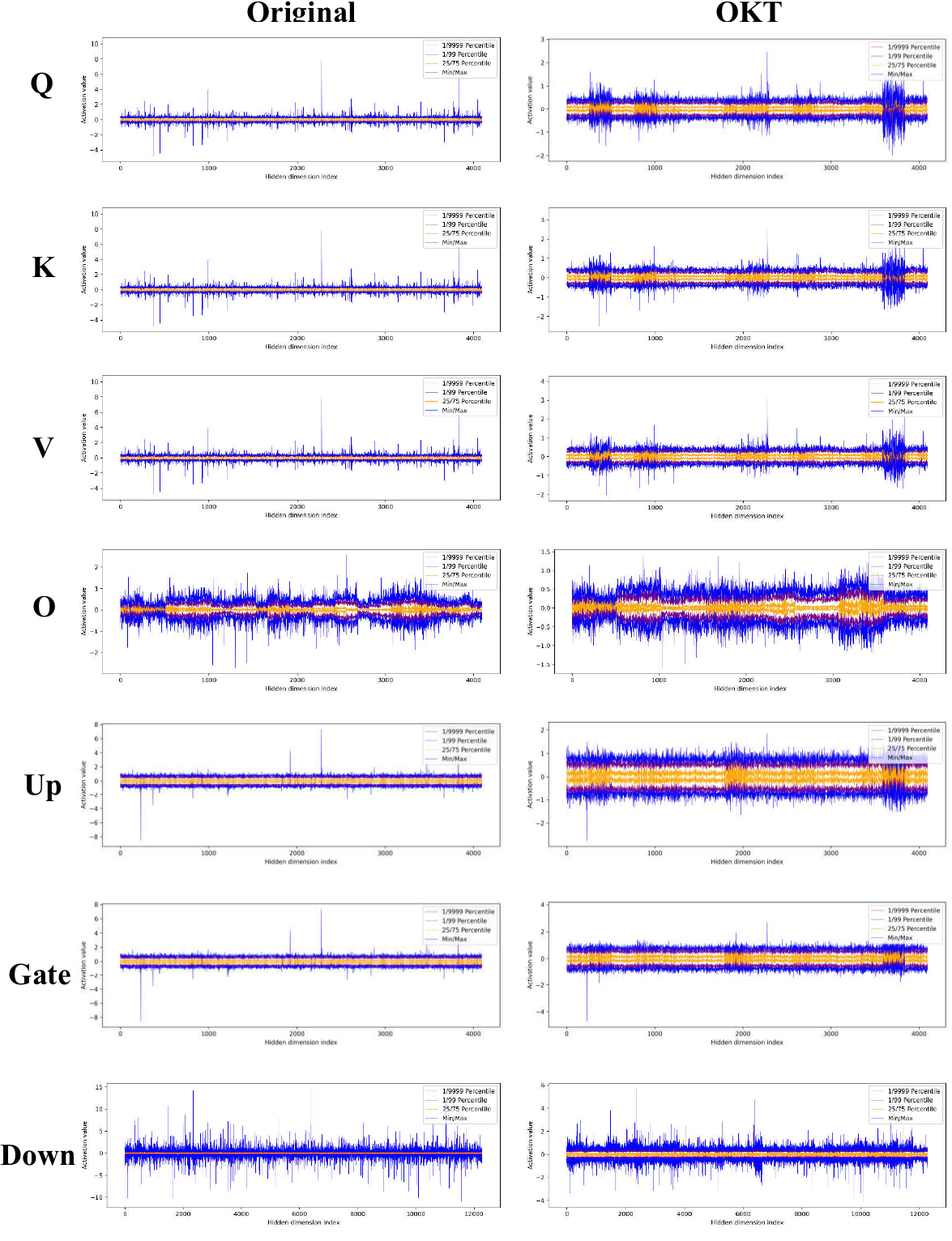}
    \caption{The activation distribution of the 12th layer in Qwen3-8B before and after BWLA.}
    \label{fig:activation_visual}
\end{figure*}